\documentclass{article}

\usepackage[nocompress]{cite}
\usepackage[pdftex]{graphicx}
\usepackage{amsmath}
\usepackage{amsthm}
\usepackage{amssymb}
\usepackage{mathtools}
\usepackage{algorithmic}
\usepackage{stfloats}
\usepackage{url}
\usepackage{enumitem}
\usepackage{booktabs}
\usepackage{multirow}
\usepackage[left=3cm, right=3cm]{geometry}

\theoremstyle{definition}

\theoremstyle{plain}
\newtheorem{theorem}{Theorem}[section]

\theoremstyle{remark}
\newtheorem{remark}{Remark}

\newcommand{\dx}{\,\mathrm{d}}
\newcommand{\C}{\mathbb{C}}
\newcommand{\R}{\mathbb{R}}
\newcommand{\N}{\mathbb{N}}
\newcommand{\E}{\mathbb{E}}
\newcommand{\Prob}{\mathbb{P}}
\newcommand{\Tmax}{T_\mathrm{max}}
\newcommand{\Id}{\mathrm{Id}}
\newcommand{\loss}{\ell}
\newcommand{\init}{\mathrm{init}}
\newcommand{\emp}{\mathrm{emp}}
\newcommand{\PSNR}{\mathrm{PSNR}}
\DeclareMathOperator*{\argmax}{argmax}

\DeclareMathOperator{\proj}{proj}

\begin{document}

\title{Total Deep Variation: A Stable Regularizer for Inverse Problems}

\author{Erich~Kobler, Alexander~Effland, Karl~Kunisch and~Thomas~Pock}

\maketitle

\begin{abstract}
Various problems in computer vision and medical imaging can be cast as inverse problems.
A frequent method for solving inverse problems is the variational approach, which amounts to minimizing an energy composed of a data fidelity term and a regularizer.
Classically, handcrafted regularizers are used, which are commonly outperformed by state-of-the-art deep learning approaches.
In this work, we combine the variational formulation of inverse problems with deep learning by introducing the data-driven general-purpose total deep variation regularizer.
In its core, a convolutional neural network extracts local features on multiple scales and in successive blocks.
This combination allows for a rigorous mathematical analysis including an optimal control formulation of the training problem in a mean-field setting
and a stability analysis with respect to the initial values and the parameters of the regularizer.
In addition, we experimentally verify the robustness against adversarial attacks and numerically derive upper bounds for the generalization error.
Finally, we achieve state-of-the-art results for numerous imaging tasks.
\end{abstract}

\section{Introduction}
\label{sec:introduction}
The statistical viewpoint of linear inverse problems accounts for measurement uncertainties and loss of information in the observations~$z$ in a rigorous framework.
Bayes' theorem states that the posterior probability~$p(x\vert z)$ is proportional to the product of the data likelihood~$p(z\vert x)$ and the prior~$p(x)$, i.e.
\[
p(x\vert z)\propto p(z\vert x)p(x),
\]
and represents the belief in a distinct solution~$x$ given~$z$.
Typically, solutions are computed by maximizing the posterior probability, which yields the maximum a posterior (MAP) estimator.
In a negative log-domain, the MAP estimator amounts to minimizing the variational problem
\[
\mathrm{E}(x,z)\coloneqq\mathrm{D}(x,z)+\mathrm{R}(x),
\]
where the data fidelity term~$\mathrm{D}$ is identified with~the negative log-likelihood~$-\log p(z\vert x)$ and the regularizer~$\mathrm{R}$ corresponds to the negative log of the prior distribution~$-\log p(x)$.
In this paper, we assume that the observations are generated by a linear inverse problem of the form
\[
z=Ay+\xi,
\]
where $y$ is the unknown ground truth, $A$ is a known task-dependent linear operator and $\xi$ is additive noise.
For example, $A$ is the identity matrix in the case of denoising, and it is a downsampling operator in the case of single image super-resolution.
While the data fidelity term is straightforward to model, there has been much effort to design a regularizer that captures the complexity of the statistics of natural images.

A classical and widely used regularizer is the total variation~(TV) originally proposed in~\cite{RuOsFa92}, which is based on the first principle assumption that images are piecewise constant with sparse gradients.
A well-known caveat of the sparsity assumption of TV is the formation of clearly visible artifacts known as staircasing effect.
To overcome this problem, the first principle assumption has later been extended to piecewise smooth images incorporating higher order image derivatives such as
infimal convolution based models~\cite{ChLi97} or the total generalized variation~\cite{BrKu10}.
Inspired by the fact that edge continuity plays a fundamental role in the human visual perception, regularizers penalizing the curvature of level lines have been proposed in~\cite{NiMu93,ChKa02,ChPo19}.
While these regularizers are mathematically well-understood, the complexity of natural images is only partially reflected in their formulation.
For this reason, handcrafted variational methods have nowadays been largely outperformed by purely data-driven methods as predicted by Levin and Nadler~\cite{LeNa11} a decade ago.

It has been recognized quite early that a proper statistical modeling of regularizers should be based on learning~\cite{ZhWu98},
which has recently been advocated e.g.~in~\cite{LuOk18,LiSc20}.
One of the most successful early approaches is the Fields of Experts (FoE) regularizer~\cite{RoBl09}, which can be interpreted as a generalization of the total variation, but builds upon learned filters and learned potential functions.
While the FoE prior was originally learned generatively, it was shown in~\cite{SaTa09} that a discriminative learning via implicit differentiation yields improved performance.
A computationally more feasible method for discriminative learning is based on unrolling a finite number of iterations of a gradient descent algorithm~\cite{Do12}.
Additionally using iteration dependent parameters in the regularizer was shown to significantly increase the performance (TNRD~\cite{ChPo17}, \cite{Le16}).
In~\cite{KoKl17}, variational networks (VNs) are proposed, which give an incremental proximal gradient interpretation of TNRD.
Interestingly, such truncated schemes are not only computationally much more efficient, but are also superior in performance with respect to the full minimization.
A continuous time formulation of this phenomenon was proposed in~\cite{EfKo19} by means of an optimal control problem, within which an optimal stopping time is learned.

An alternative approach to incorporate a regularizer into a proximal algorithm, known as plug-and-play prior~\cite{VeSi13} or regularization by denoising~\cite{RoEl17},
is the replacement of the proximal operator by an existing denoising algorithm such as BM3D~\cite{DaFo07}.
Combining this idea with deep learning was proposed in~\cite{MeMo17,RiCh17,ZhVa20}.
However, all aforementioned schemes lack a variational structure and are thus not interpretable in the framework of MAP inference.

The significance of stability for data-driven methods has recently been addressed in~\cite{AnRe20}, in which a systematical treatment of adversarial attacks for inverse problems was performed.
This issue has been studied in the context of classification by~\cite{SzZa14}, where adversarial samples have been introduced.
These samples are computed by perturbing input images as little as possible such that the attacked algorithm predicts wrong labels.
Incorporating adversarial samples in the training process of data driven methods increases their robustness as studied in \cite{SzZa14}.
In recent years, several methods were proposed for generating adversarial examples such as the fast gradient sign method~\cite{GoSh15}, Deepfool~\cite{MoSe16} or universal adversarial perturbations~\cite{MoSe17}.
In the context of inverse problems, adversarial examples are typically computed by maximizing the norm difference between the output of an algorithm and the associated ground truth in a local neighborhood around an input.
Thus, the robustness w.r.t.~adversarial attacks of an algorithm depends to a large extent on the local Lipschitz constant of the mapping defined by the algorithm.

This paper is an extended version of the prior conference proceeding~\cite{KoEf20}, in which a novel data-driven regularizer termed \emph{Total Deep Variation (TDV)} is introduced.
The TDV regularizer is inspired by recent design patterns of state-of-the-art deep convolutional neural networks and simultaneously ensures a variational structure.
We achieve this by representing the total energy of the regularizer by means of a residual multi-scale network (Figure~\ref{fig:network}) leveraging smooth activation functions.
Our proposed TDV regularizer can be used as a generic regularizer in variational formulations of linear inverse problems.
In analogy to~\cite{EfKo19,KoEf20}, we minimize an energy composed of a data fidelity term and the TDV regularizer using a gradient flow emanating from the corrupted input image on a finite time horizon~$[0,T]$ for a stopping time~$T$,
where the terminal state of the gradient flow defines the reconstructed image. 
Then, the stopping time and the parameters of the TDV regularizer are computed by solving a mean-field optimal control problem as advocated in~\cite{EHa19},
in which the cost functional is defined as the expectation of the loss function with respect to a training data distribution.
The state equation of the optimal control problem is a stochastic ordinary differential equation coinciding with the gradient flow of the energy, where the only source of randomness is the initial state.
We prove the existence of minimizers for this optimal control problem using the direct method in the calculus of variations.
A semi-implicit time discretization of the gradient flow results in a discretized optimal control problem in the mean-field setting, for which we also prove the existence of minimizers
as well as a first order necessary condition to automatize the computation of the optimal stopping time.
This training process is a form of \emph{discriminative learning} because we directly learn the functional form of the negative log-posterior~\cite{NgJo02}, in which
the TDV regularizer can be interpreted as a \emph{discriminative prior}.
In fact, the learned TDV regularizer adapts to the specific imaging task as we show in the eigenfunction analysis.
Moreover, the particular recursive structure of the discrete gradient flow allows the derivation of a stability analysis with respect to the initial states and the learned TDV parameters.
Both estimates depend on the local Lipschitz constant of the TDV regularizer, which is estimated in the mean-field setting.
Several numerical experiments demonstrate the applicability of the proposed method to numerous image restoration problems, in which we obtain state-of-the-art results.
In particular, we examine the robustness of this approach against perturbations and adversarial attacks, and an upper bound for the generalization error is empirically computed.

The major contributions of this paper are as follows:
\begin{itemize} \setlength\itemsep{0em}
\item
the design of a novel generic multi-scale variational regularizer learned from data,
\item
a rigorous mathematical analysis including a mean-field optimal control formulation of the learning problem,
\item
a stability analysis of the proposed method, which is validated by numerical experiments,
\item
a nonlinear eigenfunction analysis for the visualization and understanding of the learned regularizer,
\item
numerical evaluation of the robustness against adversarial attacks and empirical upper bounds for the generalization error,
\item
state-of-the-art results on a number of classical image restoration and medical image reconstruction problems with an impressively low number of learned parameters.
\end{itemize}
As already mentioned above, this paper extends the conference paper~\cite{KoEf20} in several aspects.
First, we extend the optimal control problem to a mean-field perspective including existence theorems.
Second, a rigorous stability analysis is pursued and numerical experiments regarding the robustness are performed.
Finally, we add several experiments involving color denoising, a multi-channel eigenfunction analysis, and numerical experiments addressing the robustness against adversarial attacks.

Let $x\in\R^n$.
We denote by $\Vert x\Vert_2\coloneqq\sqrt{\sum_{i=1}^n x_i^2}$ the $\ell^2$-norm of $x$ and by $\Vert x\Vert_\infty\coloneqq\max_{i=1,\ldots,n}\vert x_i\vert$ its $\ell^\infty$-norm.
Further, $A^\top\in\R^{n\times m}$ denotes the matrix transpose fo $A\in\R^{m\times n}$.
For two Banach spaces~$\mathcal{X}$ and~$\mathcal{Y}$, we denote by $C^0(\mathcal{X},\mathcal{Y})$ the space of continuous functions mapping from~$\mathcal{X}$ to~$\mathcal{Y}$,
and by $C^k(\mathcal{X},\mathcal{Y})$ for $k\geq 1$ the space of $k$-times continuously differentiable functions from~$\mathcal{X}$ to~$\mathcal{Y}$.
Finally, we denote by $D_s f$ the Jacobian of the function~$f$ w.r.t.~the $s^{th}$ argument.

\section{Mean-Field Optimal Control Problem}
In this section, we cast the training process as a mean-field optimal control problem taking into account the general approach presented in~\cite{EHa19}.

Let $(\mathcal{Y}\times\Xi,\mathcal{F},\Prob)$ be a complete probability space on~$\mathcal{Y}\times\Xi$ with $\sigma$-algebra~$\mathcal{F}$ and probability measure~$\Prob$.
We denote by $(y,\xi)$ a pair of independent random variables modeling the data representing the \emph{ground truth image}~$y\in\mathcal{Y}\subset\R^{nC}$ and \emph{additive noise}~$\xi\in\Xi\subset\R^l$ with associated distribution denoted by~$\mathcal{T}=\mathcal{T}_\mathcal{Y}\times\mathcal{T}_\Xi$.
Each ground truth image~$y$ represents an image of size $n=n_1\times n_2$ with $C$~channels and is related to the additive noise~$\xi$ by means of the \emph{observation}
\[
z=Ay+\xi,
\]
where $A\in\R^{l\times nC}$ is a fixed task-dependent linear operator of this linear inverse problem.
In particular, we assume that both $\mathcal{Y}$ and $\Xi$ are compact sets, which implies that all observations are contained in a compact set $\mathcal{Z}\subset\R^l$.
To estimate the unknown ground truth image~$y$ from the observation~$z$ we pursue a variational approach, which amounts to minimizing the energy functional
\begin{equation}
\mathrm{E}(x,z,\theta)\coloneqq\mathrm{D}(x,z)+\mathrm{R}(x,\theta)
\label{eq:energyFunctional}
\end{equation}
among all~$x\in\R^{nC}$.
In this paper, we consider the squared $\ell^2$-data fidelity term $\mathrm{D}(x,z)=\frac{1}{2}\Vert Ax-z\Vert_2^2$ and the total deep variation~\cite{KoEf20} regularizer~$\mathrm{R}$,
which is a convolutional neural network depending on learned training parameters~$\theta\in\Theta$, where $\Theta\subset\R^p$ is compact and convex
(the exact structure of the network is discussed below).

Let $\tilde{x}\in C^1([0,T],\R^{nC})$ be an \emph{image trajectory}, which evolves according to the \emph{gradient flow equation}~\cite{AmGi08} 
associated with~\eqref{eq:energyFunctional} on a finite time interval~$(0,T)$ given by
\begin{align}
\dot{\tilde{x}}(t)=&-D_1\mathrm{E}(\tilde{x}(t),z,\theta)=f(\tilde{x}(t),z,\theta)\notag\\
\coloneqq&-A^\top(A\tilde{x}(t)-z)-D_1\mathrm{R}(\tilde{x}(t),\theta)\label{eq:originalGradientFlow}
\end{align}
for $t\in(0,T)$ and $x(0)=x_0$.
Here, the observation-dependent initial value~$x_0$ is computed as~$x_0=A_\init z$ for a fixed task-dependent matrix $A_\init\in\R^{nC\times l}$, which could be, for instance, the pseudoinverse of $A$.
The proper choice of the stopping time $T\in[0,\Tmax]$ for a fixed~$\Tmax>0$ is essential for the quality of the reconstruction $\tilde{x}(T)$ of~$y$.
A more feasible, yet equivalent gradient flow is derived from the reparametrization $x(t)=\tilde{x}(tT)$, which yields for $t\in(0,1)$
\begin{equation}
\dot{x}(t)=Tf(x(t),z,\theta)
\label{eq:gradientFlow}
\end{equation}
with the same initial value as before.
We frequently write $x(t,y,\xi,T,\theta)$ to highlight the dependency of the image trajectory on the parameters $(y,\xi,T,\theta)\in\mathcal{Y}\times\Xi\times[0,\Tmax]\times\Theta$ for given~$t\in[0,1]$.
In particular, $x(1,y,\xi,T,\theta)$ is the computed output image.

In what follows, the training process is described as a \emph{mean-field optimal control problem}~\cite{EHa19} with control parameters~$\theta$ and~$T$.
To this end, let $\loss\in C^1(\R^{nC},\R_0^+)$ be a convex and coercive \emph{loss function}.
Then, we define the \emph{cost functional} as
\[
\mathrm{J}(T,\theta)\coloneqq\E_{(y,\xi)\sim\mathcal{T}}\left[\loss(x(1,y,\xi,T,\theta)-y)\right],
\]
which results in the mean-field optimal control problem
\begin{equation}
\inf_{T\in[0,\Tmax],\theta\in\Theta}\mathrm{J}(T,\theta).
\label{eq:objectiveFunction}
\end{equation}
\begin{remark}\label{rem:discreteOCP}
The mean-field optimal control formulation already encompasses the \emph{sampled optimal control problem}.
In detail, given a finite training set $(y^i,\xi^i)_{i=1}^N\sim\mathcal{T}^N$ drawn from the data distribution we can define the \emph{discrete probability measure} as
$\Prob(y,\xi)=\frac{1}{N}\sum_{i=1}^N\delta[y=y^i]\delta[\xi=\xi^i]$, where $\delta[s=t]=1$ if $s=t$ and $0$ otherwise.
This particular choice results in the sampled cost functional
\[
\mathrm{J}(T,\theta)=\frac{1}{N}\sum_{i=1}^N\loss(x(1,y^i,\xi^i,T,\theta)-y^i).
\]
\end{remark}

\begin{figure}
\centering
\includegraphics[width=.5\linewidth]{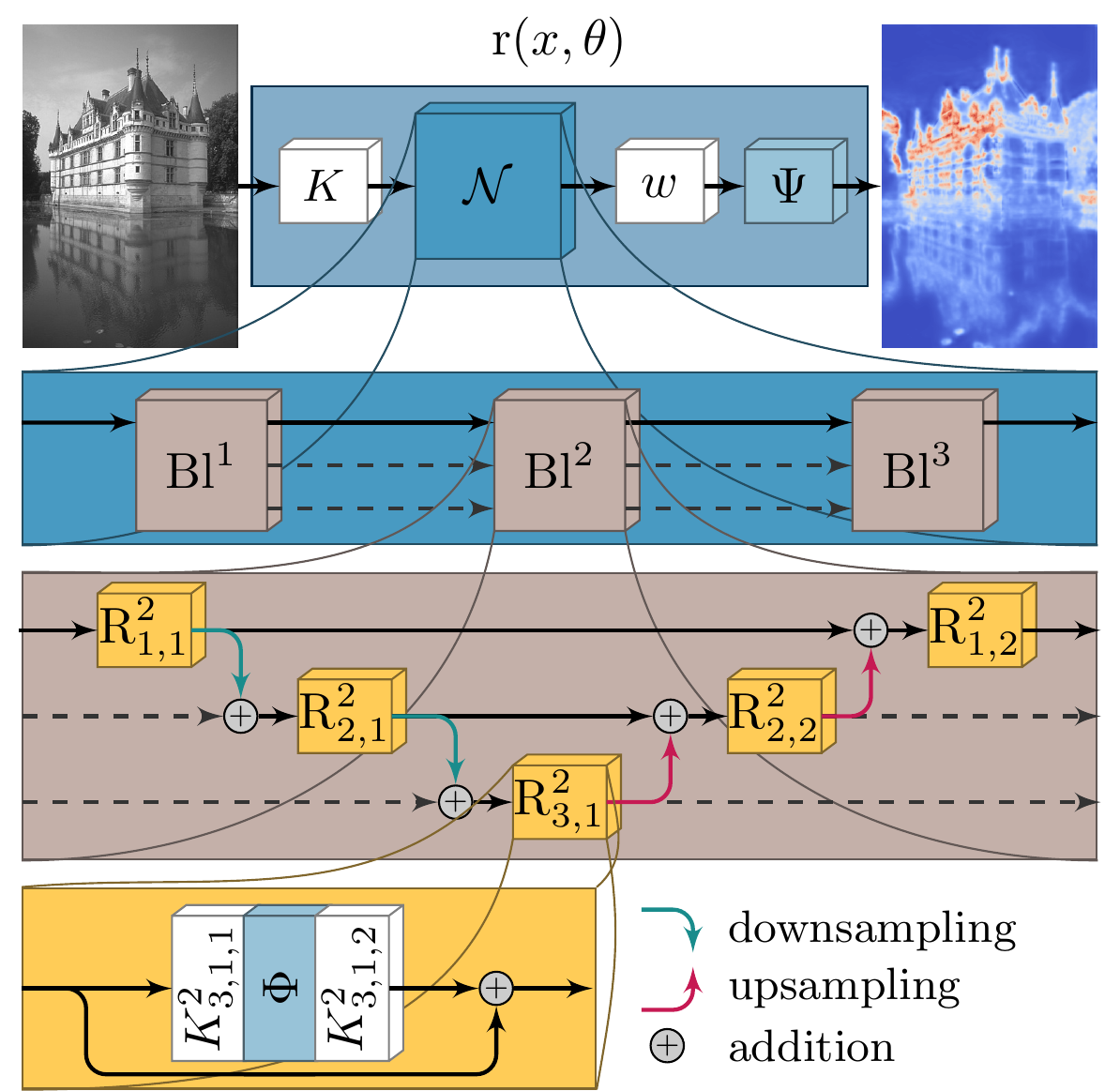}
\caption{The network structure of the total deep variation with $3$~blocks each operating on $3$~scales.}
\label{fig:network}
\end{figure}

In this paper, the regularizer~$\mathrm{R}:\R^{nC}\times\Theta\to\R$ is modeled as the total deep variation originally proposed in~\cite{KoEf20},
which is the sum of the pointwise deep variation~$\mathrm{r}:\R^{nC}\times\Theta\to\R^n$, i.e.~
\[
\mathrm{R}(x,\theta)=\sum_{i=1}^n \mathrm{r}(x,\theta)_i.
\]
The pointwise deep variation is defined as 
\[
\mathrm{r}(x,\theta)=\Psi(w\mathcal{N}(Kx)),
\]
where
\begin{itemize}[leftmargin=1.5em]
\item
$K\in\R^{nm\times nC}$ is the matrix representation of a learned $3\times3$ convolution kernel for $m=32$~feature channels with zero-mean constraint, i.e. 
$\sum_{i=1}^{nC}K_{j,i}=0$ for $j=1,\ldots,n$, which implies a spatial and radiometrical shift-invariance,
\item
$\mathcal{N}:\R^{nm}\to\R^{nm}$ is a multiscale convolutional neural network as illustrated in Figure~\ref{fig:network},
where we assume that $\Vert D\mathcal{N}\Vert_{C^0(\R^{nm})}\leq C_\mathcal{N}(\theta)$,
\item
$w\in\R^{n\times nm}$ is a learned $1\times1$ convolution kernel,
\item
$\Psi:\R^n\to\R^n,(x_1,\ldots,x_n)\mapsto(\psi(x_1),\ldots,\psi(x_n))$ using the \emph{potential function} $\psi\in C^2(\R,\R)$ satisfying $\Vert D\Psi\Vert_{C^0(\R^n)}\leq C_\Psi$ for a constant~$C_\Psi>0$.
\end{itemize}
We denote by~$\theta$ the entity of learned parameters, i.e.~$K$, $w$ and all convolution kernels of $\mathcal{N}$.
Following~\cite{KoEf20}, the total deep variation~TDV$^b_a$ for integers $a,b\geq 1$ consists of $b$~blocks $\textrm{Bl}^1,\ldots,\textrm{Bl}^b$ (gray blocks in Figure~\ref{fig:network}),
each of them has a U-Net~\cite{RoFi15} type architecture, where on all $a$~scales residual blocks $\mathrm{R}_{1,1}^i,\mathrm{R}_{1,2}^i,\ldots,\mathrm{R}_{a-1,1}^i, \mathrm{R}_{a-1,2}^i,\mathrm{R}_{a,1}^i$ (yellow blocks in Figure~\ref{fig:network}) are applied.
To increase the expressiveness of the network, residual connections are added between scales of consecutive blocks whenever possible.
Each residual block~$\mathrm{R}_{j,k}^i$ with $i\in\{1,\ldots,b\}$, $j\in\{1,\ldots,a\}$ and $k\in\{1,2\}$ exhibits the particular structure
$\mathrm{R}_{j,k}^i(x,\theta)=x+K_{j,k,2}^i\Phi(K_{j,k,1}^i x)$ for convolution operators~$K_{j,k,1}^i,K_{j,k,2}^i\in\R^{nm\times nm}$ of size $3\times 3$ with $m$~feature channels and no bias.
Following~\cite{HuMu99}, the log-student-t-distribution is a suitable model for the statistics of natural images, that is why
we choose the particular activation function $\Phi:\R^{nm}\to\R^{nm}$, $(x_1,\ldots,x_{nm})\mapsto(\phi(x_1),\ldots,\phi(x_{nm}))$ using the component-wise function $\phi(x)=\frac{1}{2}\log(1+x^2)$ with the properties $\phi'(0)=0$ and $\phi''(0)=1$.
Taking into account the work by Zhang~\cite{Zh19}, we use $3\times 3$ convolutions and transposed convolutions with stride~$2$ for downsampling and upsampling in conjunction with a blur kernel to avoid aliasing.

The existence of solutions to the mean-field optimal control problem is established in the next theorem.
\begin{theorem}\label{thm:existenceSolution}
The minimum in~\eqref{eq:objectiveFunction} is attained.
\end{theorem}
\begin{proof}
The particular structure of the proposed total deep variation results in the estimate
\begin{align*}
&\Vert D_1\mathrm{R}(x,\theta)\Vert_2=\Vert K^\top D\mathcal{N}(Kx)w^\top D\Psi(w\mathcal{N}(Kx))\Vert_2\\
\leq&\Vert K\Vert_2\Vert D\mathcal{N}(Kx)\Vert_2\Vert w\Vert_2\Vert D\Psi(w\mathcal{N}(Kx))\Vert_2\\
\leq&\Vert K\Vert_2C_\mathcal{N}(\theta)\Vert w\Vert_2C_\Psi\eqqcolon C_\mathrm{R}(\theta)
\end{align*}
for all $x\in\R^{nC}$ and all $\theta\in\Theta$,
where we used that $\Vert D\Psi\Vert_{C^0(\R^n)}\leq C_\Psi$ and $\Vert D\mathcal{N}\Vert_{C^0(\R^{nm})}\leq C_\mathcal{N}(\theta)$.
In detail, the convolutional neural network is a complex concatenation of residual blocks, where the gradient of each of these blocks can be estimated as
\begin{equation}
\Vert D_1\mathrm{R}_{j,k}^i(x,\theta)\Vert_2=\Vert\Id+(K_{j,k,1}^i)^\top D\Phi(K_{j,k,1}^i x)K_{j,k,2}^i\Vert_2
\label{eq:residualBlockBound}
\end{equation}
for $x\in\R^{nm}$.
In particular, \eqref{eq:residualBlockBound} can be uniformly bounded independently of~$x$ due to $\sup_{x\in\R}\vert\phi'(x)\vert=\frac{1}{2}$.
Then, the right-hand side of the state equation can be bounded as follows:
\begin{equation}
\Vert T f(x,z,\theta)\Vert_2\leq T(\Vert A\Vert_2\Vert z\Vert_2+C_\mathrm{R}(\theta)+\Vert A\Vert_2^2\Vert x\Vert_2)
\label{eq:upperEstimateRHS}
\end{equation}
for $z\in\R^l$.
This affine growth already ensures that the maximum domain of existence of the state equation~\eqref{eq:gradientFlow} coincides with~$\R$ \cite[Theorem~2.17]{Te12}.
As a further result, we obtain that $x\in C^1([0,1],C^0(\mathcal{Y}\times\Xi\times[0,\Tmax]\times\Theta,\R^{nC}))$ due to the smoothness of the regularizer and
\[
x(t,y,\xi,T,\theta)\in\mathcal{X}
\]
for all $(t,y,\xi,T,\theta)\in[0,1]\times\mathcal{Y}\times\Xi\times[0,\Tmax]\times\Theta$ for a compact and convex set~$\mathcal{X}\subset\R^{nC}$.

Let $(T^j,\theta^j)\in[0,\Tmax]\times\Theta$ be a minimizing sequence for~$\mathrm{J}$ with an associated state~$x^j\coloneqq x(\cdot,\cdot,\cdot,T^j,\theta^j)\in C^1([0,1],C^0(\mathcal{Y}\times\Xi,\R^{nC}))$.
The compactness of $[0,\Tmax]\times\Theta$ implies that there exists a subsequence (not relabeled) such that $(T^j,\theta^j)\to(T^\ast,\theta^\ast)\in[0,\Tmax]\times\Theta$.
In what follows, we prove that~$x^j$ converges to $x^\ast\coloneqq x(\cdot,\cdot,\cdot,T^\ast,\theta^\ast)\in C^1([0,1],C^0(\mathcal{Y}\times\Xi,\R^{nC}))$ in $C^0([0,1]\times\mathcal{Y}\times\Xi)$.
We denote by $L_x$ and $L_\theta$ the Lipschitz constants of~$D_1\mathrm{R}$ w.r.t.~$x$ and~$\theta$, i.e.
\begin{align*}
\Vert D_1\mathrm{R}(x,\theta)-D_1\mathrm{R}(\widetilde{x},\theta)\Vert_2&\leq L_x\Vert x-\widetilde{x}\Vert_2,\\
\Vert D_1\mathrm{R}(x,\theta)-D_1\mathrm{R}(x,\widetilde{\theta})\Vert_2&\leq L_\theta\Vert \theta-\widetilde{\theta}\Vert_2
\end{align*}
for all $x,\widetilde{x}\in\mathcal{X}$ and all $\theta,\widetilde{\theta}\in\Theta$.
Then, we can estimate for any~$(y,\xi)\in\mathcal{Y}\times\Xi$ and $z=Ay+\xi$ as follows:
\begin{align*}
&\Vert T^\ast f(x^\ast(t,y,\xi),z,\theta^\ast)-T^j f(x^j(t,y,\xi),z,\theta^j)\Vert_2\\
\leq&\vert T^\ast-T^j\vert\Big(\Vert A\Vert_2^2\max_{x\in\mathcal{X}}\Vert x\Vert_2+\Vert A\Vert_2\max_{z\in\mathcal{Z}}\Vert z\Vert_2+\max_{(x,\theta)\in\mathcal{X}\times\Theta}\Vert D_1\mathrm{R}(x,\theta)\Vert_2\Big)\\
&+\Tmax L_\theta\Vert\theta^\ast-\theta^j\Vert_2+\Tmax(\Vert A\Vert_2^2+L_x)\Vert x^\ast(t,y,\xi)-x^j(t,y,\xi)\Vert_2\\
\eqqcolon&C_T\vert T^\ast-T^j\vert+C_\theta\Vert\theta^\ast-\theta^j\Vert_2+C_x\Vert x^\ast(t,y,\xi)-x^j(t,y,\xi)\Vert_2.
\end{align*}
Hence, since all state equations satisfy the initial condition $x^\ast(0,z)=x^j(0,z)=A_\init z$, we 
can apply Gronwall's inequality for initial value problems~\cite[Theorem~2.8]{Te12} to obtain
\[
\Vert x^\ast(t,y,\xi)-x^j(t,y,\xi)\Vert_2\leq\frac{C_T\vert T^\ast-T^j\vert+C_\theta\Vert\theta^\ast-\theta^j\Vert_2}{C_x}\left(e^{C_xt}-1\right).
\]
Thus, we can deduce the uniform convergence of $x^j$ to $x^\ast$ in $C^0([0,1]\times\mathcal{Y}\times\Xi)$ as $j\to\infty$, which implies
$\lim_{j\to\infty}\mathrm{J}(T^j,\theta^j)=\mathrm{J}(T^\ast,\theta^\ast)$.
\end{proof}

\section{Discretization of the Optimal Control Problem}
In this section, we propose a numerical time discretization scheme for the mean-field optimal control problem discussed in the previous section.
For an a priori fixed number of iteration steps~$S\in\N$, we propose a semi-implicit discretization of the state equation~\eqref{eq:gradientFlow}, which yields
\begin{equation}
x_{s+1}=x_s-\tfrac{T}{S}A^\top(Ax_{s+1}-z)-\tfrac{T}{S}D_1\mathrm{R}(x_s,\theta)\in\R^{nC}
\label{eq:stateEquationDiscrete}
\end{equation}
for $s=0,\ldots,S-1$ and initial state $x_0=A_\init z\in\R^{nC}$.
This equation is equivalent to $x_{s+1}=g(x_s,z,T,\theta)$ with
\[
g(x,z,T,\theta)\coloneqq(\Id+\tfrac{T}{S}A^\top A)^{-1}(x+\tfrac{T}{S}(A^\top z-D_1\mathrm{R}(x,\theta))).
\]
We denote by $\widehat{x}_s(y,\xi,T,\theta)$ the state of this discretization at time~$s$ given the ground truth~$y$, the additive noise~$\xi$, the stopping time~$T$ and the parameters~$\theta$.
Note that the smoothness of the regularizer and the compactness of~$\mathcal{Y}$, $\Xi$, $[0,\Tmax]$ and~$\Theta$ directly imply that
\[
\widehat{x}:\mathcal{Y}\times\Xi\times[0,\Tmax]\times\Theta\to\mathcal{X}^{S+1}\subset(\R^{nC})^{S+1}
\]
for a \emph{compact} and convex set $\mathcal{X}$.
Then, the \emph{discretized mean-field optimal control problem} is given by
\begin{equation}
\inf_{T\in[0,\Tmax],\theta\in\Theta}\mathrm{J}_S(T,\theta),
\label{eq:discreteOptimalControl}
\end{equation}
where the discrete cost functional reads as
\[
\mathrm{J}_S(T,\theta)\coloneqq\E_{(y,\xi)\sim\mathcal{T}}\left[\loss(\widehat{x}_S(y,\xi,T,\theta)-y)\right].
\]
\begin{theorem}\label{thm:existenceSolutionDiscrete}
The minimum in~\eqref{eq:discreteOptimalControl} is attained.
\end{theorem}
\begin{proof}
Let $(T^j,\theta^j)\in[0,\Tmax]\times\Theta$ be a minimizing sequence for~$\mathrm{J}_S$ with an associated state~$\widehat{x}^j\coloneqq\widehat{x}(\cdot,\cdot,T^j,\theta^j)$.
As in the time continuous case, the compactness of $[0,\Tmax]\times\Theta$ implies the existence of a subsequence (again not relabeled) such that $(T^j,\theta^j)\to(T^\ast,\theta^\ast)\in[0,\Tmax]\times\Theta$, where
the associated state is given by $\widehat{x}^\ast\coloneqq\widehat{x}(\cdot,\cdot,T^\ast,\theta^\ast)$.
Then, we can estimate for any~$(y,\xi)\in\mathcal{Y}\times\Xi$ and $s=0,\ldots,S-1$ as follows:
\[
\Vert\widehat{x}^\ast_{s+1}(y,\xi)-\widehat{x}^j_{s+1}(y,\xi)\Vert_2
\leq C_T\vert T^\ast-T^j\vert+C_\theta\Vert\theta^\ast-\theta^j\Vert_2+C_x\Vert\widehat{x}_s^\ast(y,\xi)-\widehat{x}_s^j(y,\xi)\Vert_2.
\]
Note that the constants~$C_T$, $C_\theta$ and $C_x$ depend on $A$, $S$, $L_x$, $L_\theta$, $\Tmax$, $\Theta$ and $\mathcal{Z}$.
An induction argument reveals
\[
\Vert\widehat{x}_{s+1}^\ast(y,\xi)-\widehat{x}_{s+1}^j(y,\xi)\Vert_2\leq(C_T\vert T^\ast-T^j\vert+C_\theta\Vert\theta^\ast-\theta^j\Vert_2)\tfrac{1-C_x^{s+1}}{1-C_x}.
\]
In particular, $\Vert x_S^\ast-x_S^j\Vert_{C^0(\mathcal{Y}\times\Xi)}\to 0$ as $j\to\infty$, which implies $\lim_{j\to\infty}\mathrm{J}_S(T^j,\theta^j)=\mathrm{J}_S(T^\ast,\theta^\ast)$.
\end{proof}
The existence of the discrete adjoint state is discussed in the subsequent theorem:
\begin{theorem}\label{thm:adjoint}
Let $(T^\ast,\theta^\ast)\in[0,\Tmax]\times\Theta$ be a pair of control parameters for $\mathrm{J}_S$ and $x^\ast\in\mathcal{L}\coloneqq L^2(\mathcal{Y}\times\Xi,(\R^{nC})^{S+1})$ the corresponding state.
Then there exists a discrete adjoint state $p^\ast\in\mathcal{L}$ given by
\begin{equation}
p^\ast_s(y,\xi)=(\Id-\tfrac{T^\ast}{S}D_1^2\mathrm{R}(x^\ast_s(y,\xi),\theta^\ast))(\Id+\tfrac{T^\ast}{S}A^\top A)^{-1}p^\ast_{s+1}(y,\xi)\label{eq:adjointEquation}
\end{equation}
for $s=S-1,\ldots,0$ with terminal condition $p^\ast_S(y,\xi)=-D\loss(x^\ast_S(y,\xi)-y)$.
\end{theorem}
\begin{proof}
First, we define the functional $G:\mathcal{L}\times[0,\Tmax]\times\Theta\to\mathcal{L}$ representing the constraints as follows:
\begin{align*}
G(x,T,\theta)(y,\xi)=
\begin{pmatrix}
x_0(y,\xi)-A_\init(Ay+\xi)\\
x_1(y,\xi)-g(x_0(y,\xi),Ay+\xi,T,\theta)\\
\vdots\\
x_S(y,\xi)-g(x_{S-1}(y,\xi),Ay+\xi,T,\theta)
\end{pmatrix}.
\end{align*}
Then, the Lagrange functional~$\mathrm{L}:\mathcal{L}\times[0,\Tmax]\times\Theta\times\mathcal{L}\to\R$ using $\mathcal{L}^\ast\cong\mathcal{L}$ is given by
\[
\mathrm{L}(x,T,\theta,p)\coloneqq\E_{(y,\xi)\sim\mathcal{T}}\Bigg[\loss(x_{S}(y,\xi)-y)+\sum_{s=0}^{S}\left\langle p_s(y,\xi),G_s(x(y,\xi),T,\theta)\right\rangle\Bigg].
\]
Following~\cite[Theorem~43.D]{Ze85}, the Lagrange multiplier $p^\ast\in\mathcal{L}$ associated with $(x^\ast,T^\ast,\theta^\ast)$ exists if $\loss$ and $G$ are (continuously) Frech\'et differentiable
and $D_1 G(x^\ast,T^\ast,\theta^\ast)$ is surjective.
The differentiability requirements are immediately implied by the smoothness requirements of $\mathrm{R}$.
To prove the surjectivity of~$D_1 G$, we first compute for $x\in\mathcal{L}$
\[
D_1 G(x^\ast,T^\ast,\theta^\ast)(x)(y,\xi)
=
\begin{pmatrix}
x_0(y,\xi)\\
x_1(y,\xi)-D_1g(x^\ast_0(y,\xi),Ay+\xi,T^\ast,\theta^\ast)x_0(y,\xi)\\
\vdots\\
x_S(y,\xi)-D_1g(x^\ast_{S-1}(y,\xi),Ay+\xi,T^\ast,\theta^\ast)x_{S-1}(y,\xi)
\end{pmatrix}.
\]
Thus, for any $w\in\mathcal{L}$ the solution~$x\in\mathcal{L}$ of the equation $D_1 G(x^\ast,T^\ast,\theta^\ast)(x)=w$ is given by 
\begin{align*}
x_0(y,\xi)&=w_0(y,\xi),\\
x_s(y,\xi)&=w_s(y,\xi)+D_1g(x^\ast_{s-1}(y,\xi),Ay+\xi,T^\ast,\theta^\ast)x_{s-1}(y,\xi)
\end{align*}
for $s=1,\ldots,S$, which proves the surjectivity and thus the existence of Lagrange multipliers.
Finally, \eqref{eq:adjointEquation} is implied by the optimality of~$\mathrm{L}$ w.r.t.~$x$.
\end{proof}
Next, we derive an optimality condition for the stopping time, which can easily be evaluated numerically.
\begin{theorem}\label{thm:firstOrderConditionDiscrete}
Let $(T^\ast,\theta^\ast)$ be a stationary point of~$\mathrm{J}_S$ with associated state~$x^\ast$ and adjoint state $p^\ast$ as in Theorem~\ref{thm:adjoint}.
Then,
\begin{equation}
\E_{(y,\xi)\sim\mathcal{T}}\bigg[\sum_{s=0}^{S-1}\Big\langle p^\ast_{s+1}(y,\xi),(\Id+\tfrac{T^\ast}{S}A^\ast A)^{-1}(x^\ast_{s+1}(y,\xi)-x^\ast_s(y,\xi))\Big\rangle\bigg]=0.
\label{eq:optimalTDiscrete}
\end{equation}
\end{theorem}
\begin{proof}
Let us define $B(T)\coloneqq\Id+\frac{T}{S}A^\ast A$ and observe that
\[
\tfrac{\dx}{\dx T}(B(T)^{-1})=-B(T)^{-1}\left(\tfrac{\dx}{\dx T}B(T)\right)B(T)^{-1}.
\]
The derivative of~$g$ w.r.t.~$T$ reads as
\begin{align*}
\tfrac{\dx}{\dx T}g(x,z,T,\theta)
=&-B(T)^{-1}\Big(\tfrac{1}{S}A^\top AB(T)^{-1}(x+\tfrac{T}{S}(A^\top z-D_1\mathrm{R}(x,\theta)))-\tfrac{1}{S}(A^\top z-D_1\mathrm{R}(x,\theta))\Big)\\
=&-\tfrac{1}{T}B(T)^{-1}\Big(x-B(T)^{-1}(x+\tfrac{T}{S}(A^\top z-D_1\mathrm{R}(x,\theta)))\Big).
\end{align*}
Due to~\eqref{eq:stateEquationDiscrete} the following relation holds true for the optimal $x^\ast\in\mathcal{L}$ and $s=0,\ldots,S-1$:
\[
B(T)x_{s+1}^\ast=x_s^\ast+\tfrac{T}{S}(A^\top z-D_1\mathrm{R}(x_s^\ast,\theta)).
\]
Hence, the optimality condition of~$L$ w.r.t.~$T^\ast$ reads as
\[
\E_{(y,\xi)\sim\mathcal{T}}\bigg[-\frac{1}{T^\ast}\sum_{s=0}^{S-1}\Big\langle p^\ast_{s+1}(y,\xi),(\Id+\tfrac{T^\ast}{S}A^\ast A)^{-1}(x^\ast_{s+1}(y,\xi)-x^\ast_s(y,\xi))\Big\rangle\bigg]=0,
\]
which proves this theorem.
\end{proof}

\section{Stability Analysis}
Here, we examine the stability of the proposed method, which quantifies the changes in the output caused by local perturbations of the observations and training parameters, respectively.
The central assumption in both cases is that the distribution of the test data coincides with the distribution of the training data in the mean-field setting.
Numerical results for the stability analysis are presented in subsection~\ref{sub:stabilityExperiments}.

\subsection{Stability Analysis w.r.t.~Input}
In what follows, we perform a stability analysis for the proposed algorithm, in which we derive upper bounds along the trajectories for different noise instances in the mean-field context.
To this end, we first compute quantiles of the Lipschitz constant of the explicit update for the proposed discretization scheme given the data distribution~$\mathcal{T}$.
Then, upper bounds for the difference of trajectories associated with one ground truth image and different noise instances drawn from the data distribution are derived using a recursion argument.

Let $x,\widetilde{x}\in\R^{nC}$, $T\in[0,\Tmax]$, and $\theta\in\Theta$.
We define the local Lipschitz constant of the explicit update step~$x\mapsto x-\frac{T}{S}D_1\mathrm{R}(x,\theta)$ as
\begin{equation*}
L_x(x,\widetilde{x},T,\theta)
\coloneqq\frac{\Vert x-\frac{T}{S}D_1\mathrm{R}(x,\theta)-\widetilde{x}+\frac{T}{S}D_1\mathrm{R}(\widetilde{x},\theta)\Vert_2}{\Vert x-\widetilde{x}\Vert_2},
\end{equation*}
where we set $L_x=0$ if the denominator vanishes.
Then, the cumulative distribution function~$F_S$ of the local Lipschitz constant on the data distribution~$\mathcal{T}$ for $L\in\R$ is defined as
\[
F_S(L)=\Prob\Big(\max_{s=0,\ldots,S}L_x(\widehat{x}_s(y,\xi,T,\theta),\widehat{x}_s(y,\widetilde{\xi},T,\theta),T,\theta)\leq L:y\sim\mathcal{T}_\mathcal{Y},\,\xi,\widetilde{\xi}\sim\mathcal{T}_\Xi\Big).
\]
Thus, the maximum local Lipschitz constant of the explicit update step along each trajectory is bounded by $F_S^{-1}(1-\delta)$ with probability~$1-\delta$.
\begin{theorem}[Stability w.r.t.~input]\label{thm:stability}
Let $(T,\theta)\in[0,\Tmax]\times\Theta$ be fixed control parameters, $y\sim\mathcal{T}_\mathcal{Y}$ and $\xi,\widetilde{\xi}\sim\mathcal{T}_\Xi$.
We denote by~$x,\widetilde x\in(\R^{nC})^{S+1}$ two solutions of the state equation associated with $z=Ay+\xi$ and $\widetilde{z}=Ay+\widetilde{\xi}$, and corresponding~$x_0=A_\init z$ and $\widetilde{x}_0=A_\init\widetilde{z}$, respectively.
The discrete state equations are given by
\[
x_{s+1}=g(x_s,z,T,\theta),\quad
\widetilde{x}_{s+1}=g(\widetilde{x}_s,\widetilde{z},T,\theta)
\]
for $s=0,\ldots,S-1$.
Let $\delta\in[0,1)$,
\[
\alpha_1(\delta)\coloneqq\Vert B^{-1}\Vert_2 F_S^{-1}(1-\delta),\quad\beta_1\coloneqq\tfrac{T}{S}\Vert B^{-1}\Vert_2\Vert A\Vert_2
\]
for $B\coloneqq\Id+\tfrac{T}{S}A^\top A$.
Then,
\[
\tfrac{1}{nC}\Vert x_{s+1}-\widetilde{x}_{s+1}\Vert_2\leq\tfrac{1}{nC}\left(\alpha_1(\delta)^{s+1}\Vert A_\init\Vert_2+\tfrac{1-\alpha_1(\delta)^{s+1}}{1-\alpha_1(\delta)}\beta_1\right)\Vert z-\widetilde{z}\Vert_2
\]
holds true with probability~$1-\delta$.
\end{theorem}
\begin{proof}
The definition of the semi-implicit scheme~\eqref{eq:stateEquationDiscrete} implies that for any $s=0,\ldots,S-1$ the inequality
\[
\Vert x_{s+1}-\widetilde{x}_{s+1}\Vert_2
\leq\Vert B^{-1}\Vert_2\left(F_S^{-1}(\delta)\Vert x_s-\widetilde{x}_s\Vert_2+\tfrac{T}{S}\Vert A\Vert_2\Vert z-\widetilde{z}\Vert_2\right)
\]
holds true with probability~$1-\delta$.
By taking into account a recursion argument, the geometric series formula $\sum_{i=0}^n q^i=\frac{1-q^{n+1}}{1-q}$, and the estimate $\Vert x_0-\widetilde{x}_0\Vert_2\leq\Vert A_\init\Vert_2\Vert z-\widetilde{z}\Vert_2$ we obtain the desired result.
\end{proof}

\subsection{Stability Analysis w.r.t.~Parameters}
Next, we elaborate on the stability of the proposed approach w.r.t.~variations of the learned parameters~$\theta\in\Theta$.
To this end, we estimate the local Lipschitz constants of the TDV regularizer w.r.t.~both of its arguments in the mean-field setting to derive
upper bounds along the trajectories emanating from the same initial state, but with different parameters~$\theta$ and~$\widetilde{\theta}$.
In detail, the perturbed parameters~$\widetilde{\theta}$ is drawn from a uniform distribution supported on a component-wise relative $\epsilon$-ball around $\theta$.
A recursion argument involving the estimated Lipschitz constants results in computable upper bounds for the norm difference along trajectories associated with~$\theta$ and~$\widetilde{\theta}$.

Let $B_{\epsilon}(\theta)$ be the component-wise relative $\epsilon$-ball around $\theta=(K,K_{j,k,1}^i,K_{j,k,2}^i,w)\in\Theta$ w.r.t~the $\ell^\infty$-norm, i.e.
\begin{align*}
B_{\epsilon}(\theta)=\Big\{&
\widetilde{\theta}=(\widetilde{K},\widetilde{K}_{j,k,1}^i,\widetilde{K}_{j,k,2}^i,\widetilde{w})\in\Theta:\\
&\Vert\widetilde{K}-K\Vert_\infty\leq\epsilon\Vert K\Vert_\infty,
\Vert\widetilde{K}_{j,k,1}^i-K_{j,k,1}^i\Vert_\infty\leq\epsilon\Vert K_{j,k,1}^i\Vert_\infty,\\
&\Vert\widetilde{K}_{j,k,2}^i-K_{j,k,2}^i\Vert_\infty\leq\epsilon\Vert K_{j,k,2}^i\Vert_\infty,
\Vert\widetilde{w}-w\Vert_\infty\leq\epsilon\Vert w\Vert_\infty
\Big\}.
\end{align*}
Further, we denote by $\proj_\Theta:\R^p\to\Theta$ the orthogonal projection onto~$\Theta$, and by~$\mathcal{U}(S)$ the uniform distribution for any set~$S\subset\R^p$.
Then, the cumulative distribution function $F_{S,x}$ of the local Lipschitz constant of the regularizer w.r.t.~its first component is given as
\[
F_{S,x}(L)=\Prob\Big(\max_{s=0,\ldots,S}L_x(\widehat{x}_s(y,\xi,T,\theta),\widehat{x}_s(y,\xi,T,\widetilde{\theta}),T,\theta)\leq L:
(y,\xi)\sim\mathcal{T},\widetilde{\theta}\sim\mathcal{U}(\proj_\Theta(B_{\epsilon}(\theta)))\Big)
\]
for $L\in\R$.
Likewise, we define the local Lipschitz constant of TDV w.r.t.~its second argument as
\[
F_{S,\theta}(L)=\Prob\Big(\max_{s=0,\ldots,S}L_\theta(\widehat{x}_s(y,\xi,T,\widetilde{\theta}),\theta,\widetilde{\theta})\leq L:
(y,\xi)\sim\mathcal{T},\widetilde{\theta}\sim\mathcal{U}(\proj_\Theta(B_{\epsilon}(\theta)))\Big)
\]
for $L\in\R$, where
\[
L_\theta(x,\theta,\widetilde{\theta})
\coloneqq\frac{\Vert D_1\mathrm{R}(x,\theta)-D_1\mathrm{R}(x,\widetilde{\theta})\Vert_2}{\Vert\theta-\widetilde{\theta}\Vert_2}.
\]   
Taking into account the above definitions we can state the stability theorem w.r.t. the learned parameters as follows: 
\begin{theorem}[Stability w.r.t.~parameters]\label{thm:stabilityParameters}
Let $T\in[0,\Tmax]$, $\theta\in\Theta$ and $\widetilde{\theta}\sim\mathcal{U}(\proj_\Theta(B_{\epsilon}(\theta)))$.
We denote by $z=Ay+\xi$ an observation associated with $(y,\xi)\sim\mathcal{T}$, and
by $\{x_s\}_{s=0}^S,\{\widetilde{x}_s\}_{s=0}^S\in(\R^{nC})^{S+1}$ two states satisfying~\eqref{eq:stateEquationDiscrete} with initial conditions $x_0=\widetilde{x}_0=A_\init z$ and control parameters $(T,\theta)$ and $(T,\widetilde{\theta})$, respectively.
Then, the inequality
\begin{equation}
\frac{1}{nC}\Vert x_{s+1}-\widetilde{x}_{s+1}\Vert_2\leq\frac{1}{nC}\frac{1-\alpha_2(\delta)^{s+1}}{1-\alpha_2(\delta)}\beta_2(\delta)\Vert\theta-\widetilde{\theta}\Vert_2
\label{eq:stabilityParameters}
\end{equation}
holds true with probability~$1-\delta$ for $\delta\in[0,1)$, where
\[
\alpha_2(\delta)=\Vert B^{-1}\Vert_2 F_{S,x}^{-1}(1-\tfrac{\delta}{2}),\quad\beta_2(\delta)=\Vert B^{-1}\Vert_2\tfrac{T}{S}F_{S,\theta}^{-1}(1-\tfrac{\delta}{2})
\]
for $B\coloneqq\Id+\frac{T}{S}A^\top A$.
\end{theorem}
\begin{proof}
Again, using the definition of~$g$ yields
\[
\Vert x_{s+1}-\widetilde{x}_{s+1}\Vert_2
\leq\alpha_2(\delta)\Vert x_s-\widetilde{x}_s\Vert_2+\beta_2(\delta)\Vert\theta-\widetilde{\theta}\Vert_2
\]
with probability~$1-\delta$, where we exploited
\[
\Vert\mathrm{R}(x_s,\theta)-\mathrm{R}(\widetilde{x}_s,\widetilde{\theta})\Vert_2\leq F_{S,x}^{-1}(1-\tfrac{\delta}{2})\Vert x_s-\widetilde{x}_s\Vert_2+F_{S,\theta}^{-1}(1-\tfrac{\delta}{2})\Vert \theta-\widetilde{\theta}\Vert_2.
\]
By exploiting a recursion argument and noting that the initial states coincide the theorem follows.
\end{proof}
Hence, this theorem provides a computable upper bound for the norm difference of two states w.r.t.~perturbations of the TDV parameters.
In particular, if $(T,\theta)$ is a local minimizer of the cost functional~\eqref{eq:discreteOptimalControl}, then the stability analysis quantifies the robustness of the trajectories.

\section{Numerical Results}
In this section, we present numerical results for additive Gaussian denoising, medical reconstruction, and single image super-resolution.
To get an intuition for the local behavior of the learned TDV regularizer, we pursue a nonlinear eigenfunction analysis.
Moreover, we perform a stability analysis including adversarial attacks and worst case generalization error estimates to demonstrate the robustness of the proposed method.

\subsection{Training Details}
In all experiments, we use the BSDS400 dataset~\cite{MaFo01} for training, which determines the discrete probability measure according to Remark~\ref{rem:discreteOCP}.
Thus, the control parameters $(T,\theta)$ are computed by minimizing the \emph{discretized sampled optimal control problem}
\[
\min_{T\in[0,\Tmax],\theta\in\Theta}\frac{1}{N}\sum_{i=1}^N\loss(\widehat{x}_S(y^i,\xi^i,T,\theta)-y^i),
\]
where $\loss(x)=\Vert x\Vert_2^2$ for Gaussian denoising and $\loss(x)=\sum_{i=1}^{nC}\sqrt{x_i^2+\epsilon^2}$ for single image super-resolution with $\epsilon=0.01$.
We augment data of patch size $93\times 93$ by randomly flipping the images horizontally or vertically, and by rotating the images by multiples of $90^\circ$.
The ADAM optimizer~\cite{KiBa15} is employed with a mini batch size of~$32$ using $10^5$ iterations, $\beta_1=0.9$ and $\beta_2=0.999$, where the initial learning rate $4\cdot 10^{-4}$ is halved every $25000$ iterations.
The noise~$\xi$ is drawn randomly in each iteration.

\subsection{Additive Gaussian Denoising}
As a first task, we consider additive Gaussian denoising implying $A=\Id\in\R^{nC\times nC}$, $\xi\sim\mathcal{N}(0,\sigma^2\Id)$ and $l=nC$ for $C=1$ (gray-scale images) or $C=3$ (color images).

In the first experiments, we perform an ablation study of the number of blocks~$b$, the number of scales~$a$ and the potential function.
To this end, we evaluate the performance of the resulting TDV regularizers for additive gray-scale Gaussian denoising by computing the expected $\PSNR$ value on the BSDS68 dataset.
Figure~\ref{fig:DenoisingS_psnr_T} depicts the expected $\PSNR$ values (top) and the optimal stopping times (bottom) as functions of the depth~$S$ for color-coded TDV regularizers with $a,b\in\{2,3,4\}$.
In all cases, the performance increases until $S\approx 10$, beyond this point the curves saturate.
Thus, in all subsequent experiments the TDV regularizer is trained for $S=10$.
Moreover, the expected $\PSNR$ values increase with the number of learnable parameters, which is correlated with the number of blocks and scales.
However, beyond a certain complexity the performance increase saturates, that is why we use the TDV$_3^3$ regularizer in all further experiments.

Table~\ref{tab:potentialFunction} lists the $\PSNR$ values for three possible choices of the potential functions.
It turns out that the simplest potential function $\psi(x)=x$, which is neither bounded nor coercive, performs slightly better than the other potential functions.
For this reason, we use $\psi(x)=x$ in all further experiments.
\begin{table}
\caption{Different possible choices for potential functions~$\psi$ evaluated on gray-scale Gaussian denoising ($\sigma=25$).}
\label{tab:potentialFunction}
\centering
\begin{tabular}{r | c c c}
\toprule[1pt]
$\psi$ & $\ln\cosh(x)$ & $\frac{1}{2}\log(1+x^2)$ & $x$ \\ \midrule[.5pt]
$\PSNR$  & 29.3596 & 29.3662 & 29.3722\\
\bottomrule[1pt]
\end{tabular}
\end{table}
\begin{figure}
\centering
\includegraphics[width=.7\linewidth]{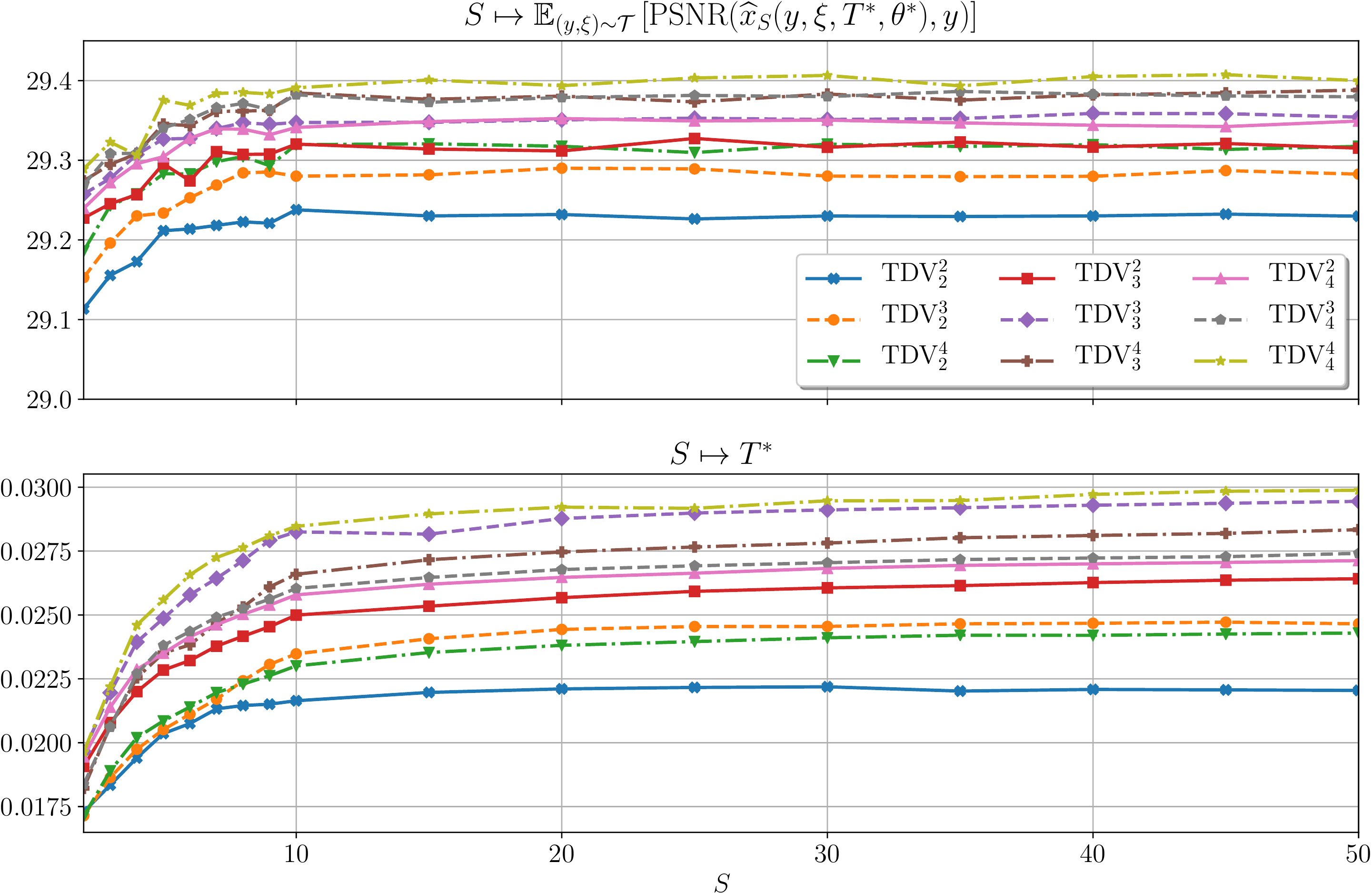}
\caption{Expected PNSR value and optimal stopping time depending on~$S$ for various TDV regularizers (gray-scale Gaussian denoising, $\sigma=25$).}
\label{fig:DenoisingS_psnr_T}
\end{figure}

In what follows, we discuss the importance of the stopping time for the quality of the output image.
To this end, we plot the $\PSNR$ values of all BSDS68 test images and the corresponding expected $\PSNR$ value (red line) as a function of the stopping time (Figure~\ref{fig:DenoisingOptimalStoppingTime}, top).
All curves approximately peak around the same optimal stopping time~$T^\ast=0.172$, which is also identified by the first order condition of Theorem~\ref{thm:firstOrderConditionDiscrete} (Figure~\ref{fig:DenoisingOptimalStoppingTime}, bottom).
Further, Figure~\ref{fig:DenoisingOptimalStoppingSequence} presents sequences of output images for gray-scale and color Gaussian denoising trained for $S=10$ to visually verify the importance of the proper choice of the optimal stopping time.
Starting from the noisy input image~$x_0$ (second column), the noise level is gradually decreased until the output image~$x_{10}$ (fourth column) is obtained.
Beyond this point, the algorithm generates oversmoothed images and details are lost.
\begin{figure}
\centering
\includegraphics[width=.7\linewidth]{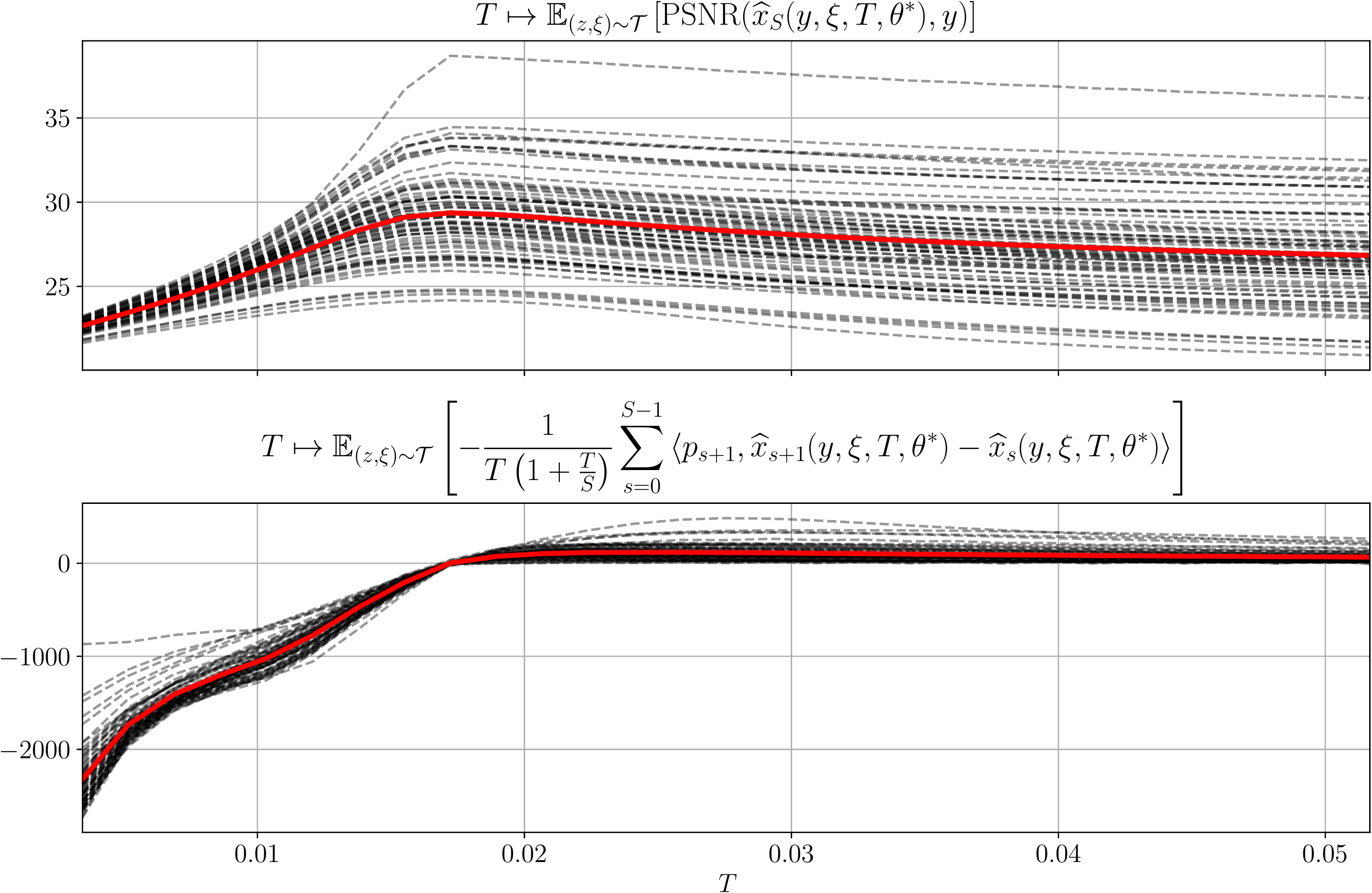}
\caption{First order optimality condition of the stopping time for gray-scale Gaussian denoising ($\sigma=25$) using TDV$_3^3$ and the BSDS68 dataset.}
\label{fig:DenoisingOptimalStoppingTime}
\end{figure}

\begin{figure}
\centering

\includegraphics[width=\linewidth]{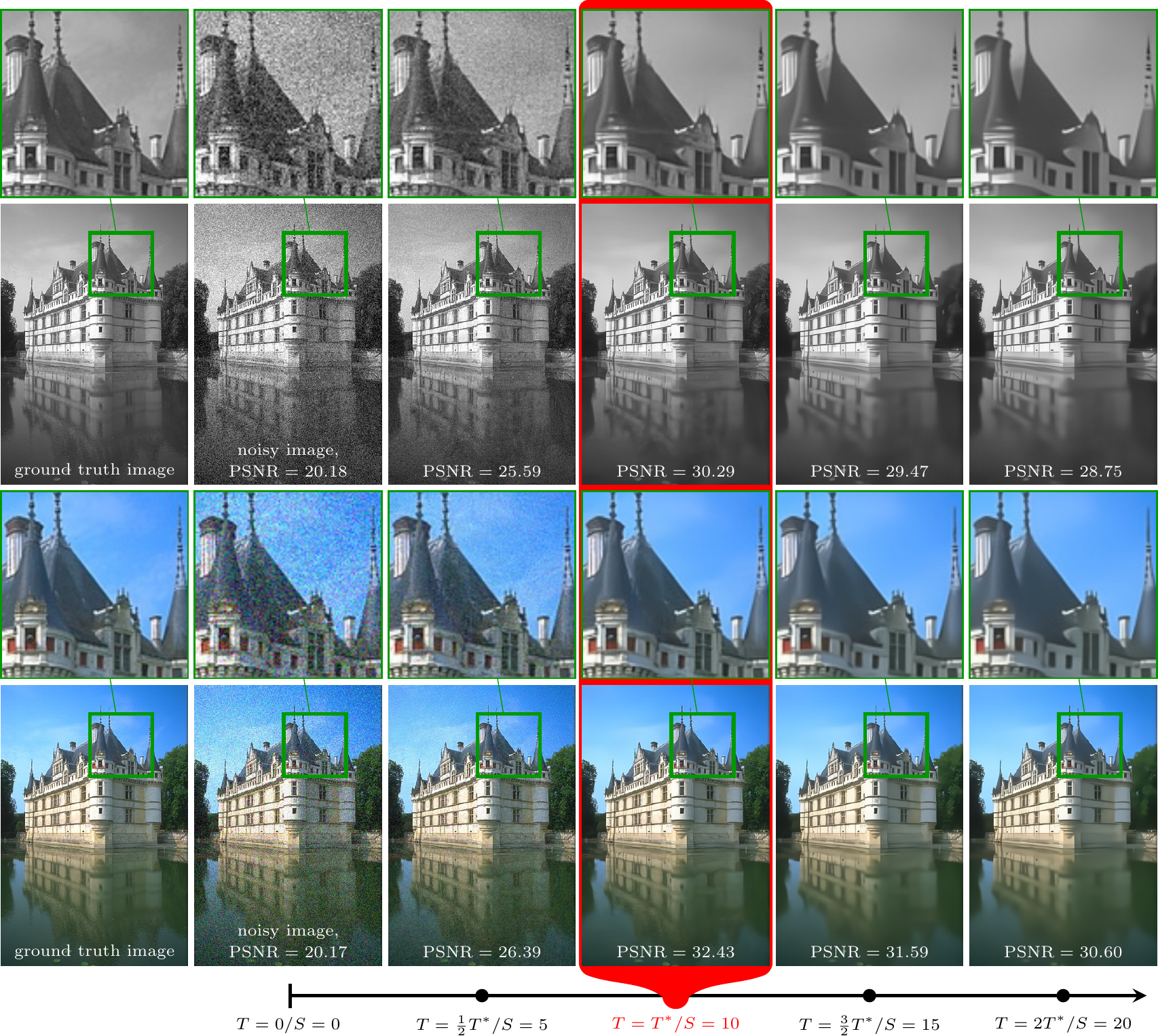}

\caption{From left to right: Ground truth, noisy input with noise level~$\sigma=25$ and resulting output of TDV$_3^3$ for $(S,T)\in\{(5,\frac{1}{2}T^\ast),(10,T^\ast),(15,\frac{3}{2}T^\ast),(20,2T^\ast)\}$
for Gaussian denoising of gray-scale images (top) and color images (bottom).
Note that the best images are framed in red and are obtained at the optimal stopping time, which is $T^\ast=0.172$ for gray-scale denoising and $T^\ast=0.0247$ for color denoising.}
\label{fig:DenoisingOptimalStoppingSequence}
\end{figure}

A quantitative comparison of expected $\PSNR$ values for additive gray-scale and color Gaussian denoising for~$\sigma\in\{15,25,50\}$ on various image datasets is listed in Table~\ref{tab:awgn} and Table~\ref{tab:awgnColor}. 
In the TDV$_{3,25}^3$ column, the $\PSNR$ values of our proposed TDV regularizer with three macro-blocks on three scales solely trained for $\sigma=25$ are presented. 
To apply the TDV$^3_{3,25}$ model to different noise levels, we first rescale the noisy images~$\overline{x}_\init=\overline{z}=\tfrac{25}{\sigma}z$, then apply the learned scheme~\eqref{eq:stateEquationDiscrete}, and obtain the results via~$x_S=\tfrac{\sigma}{25}\overline{x}_S$.
In the last column of Table~\ref{tab:awgn}, the $\PSNR$ values of the TDV regularizer--\emph{individually} trained for each noise level--are listed.
For color Gaussian denoising we only present results obtained by TDV$_{3,25}^3$ to follow the evaluation standard of the related methods.
We achieve state-of-the-art results for gray-scale and color image denoising compared with models of similar complexity.
Only FOCNet~\cite{JiLi19} performs slightly better for gray-scale images at the expense of more than hundred times more trainable parameters.
Finally, the TDV regularizers yield higher $\PSNR$ values if its parameters are individually optimized for each noise level.
\begin{table}
\caption{Comparison of expected $\PSNR$ values for additive gray-scale Gaussian denoising for~$\sigma\in\{15,25,50\}$ on various image datasets.}
\label{tab:awgn}
\centering
\resizebox{\linewidth}{!}{
\begin{tabular}{l c*{7}{c} c}
\toprule[1.5pt]
Dataset & $\sigma$ & BM3D~\cite{DaFo07} & TNRD~\cite{ChPo17} & DnCNN~\cite{ZhZu17} & FFDNet~\cite{ZhZu18} & N$^3$Net~\cite{PlRo18} & FOCNet~\cite{JiLi19} & TDV$_{3,25}^3$ & TDV$_3^3$ \\ \midrule[1pt]
\multirow{3}{*}{Set12} 
& 15 & 32.37 & 32.50 & 32.86 & 32.75 &  -    & 33.07 & 32.93 & 33.02\\
& 25 & 29.97 & 30.05 & 30.44 & 30.43 & 30.55 & 30.73 & 30.68 & 30.68\\
& 50 & 26.72 & 26.82 & 27.18 & 27.32 & 27.43 & 27.68 & 27.52 & 27.59\\ \midrule
\multirow{3}{*}{BSDS68} 
& 15 & 31.08 & 31.42 & 31.73 & 31.63 &  -    & 31.83 & 31.76 & 31.84\\
& 25 & 28.57 & 28.92 & 29.23 & 29.19 & 29.30 & 29.38 & 29.37 & 29.37\\
& 50 & 25.60 & 25.97 & 26.23 & 26.29 & 26.39 & 26.50 & 26.42 & 26.45\\ \midrule
\multirow{3}{*}{Urban100}
& 15 & 32.34 & 31.98 & 32.67 & 32.43 &  -    & 33.15 & 32.66 & 32.91 \\
& 25 & 29.70 & 29.29 & 29.97 & 29.92 & 30.19 & 30.64 & 30.38 & 30.38 \\
& 50 & 25.94 & 25.71 & 26.28 & 26.52 & 26.82 & 27.40 & 26.94 & 27.04 \\ \midrule[1pt]
\# Parameters
&    &       & 26,645 & 555,200 & 484,800 & 705,895 & 53,513,120 & 387,394 & 387,394\\
\bottomrule[1.5pt]
\end{tabular}
}
\end{table}
\begin{table}
\caption{Comparison of expected $\PSNR$ values for additive color Gaussian denoising for~$\sigma\in\{15,25,50\}$ on various image datasets.}
\label{tab:awgnColor}
\centering
\begin{tabular}{l c*{4}{c} c}
\toprule[1.5pt]
Dataset & $\sigma$ & BM3D~\cite{DaFo07} & CDnCNN~\cite{ZhZu17} & FFDNet~\cite{ZhZu18} & TDV$_{3,25}^3$ \\ \midrule[1pt]
\multirow{3}{*}{CBSDS68} 
& 15 & 33.52 & 33.89 & 33.87 & 34.12 \\
& 25 & 30.71 & 31.23 & 31.21 & 31.53 \\
& 50 & 27.38 & 27.92 & 27.96 & 28.26 \\ \midrule
\multirow{3}{*}{Kodak24} 
& 15 & 34.28 & 34.48 & 34.63 & 35.01 \\
& 25 & 31.68 & 32.03 & 32.13 & 32.59 \\
& 50 & 28.46 & 28.85 & 28.98 & 29.44 \\ \midrule
\multirow{3}{*}{McMaster}
& 15 & 34.06 & 33.44 & 34.66 & 34.55 \\
& 25 & 31.66 & 31.51 & 32.35 & 32.47 \\
& 50 & 28.51 & 28.61 & 29.18 & 29.41 \\ \midrule[1pt]
\# Parameters
&    &       & 668,803 & 852,108 & 387,970\\
\bottomrule[1.5pt]
\end{tabular}
\end{table}

\subsection{Computed Tomography Reconstruction}
To demonstrate the broad applicability of the proposed TDV regularizer, we perform a two-dimensional computed tomography (CT) reconstruction
using the TDV$_{3,25}^3$ regularizer, which was trained for gray-scale Gaussian image denoising and $S=10$.
We stress that the regularizer is applied \emph{without} any additional training of the TDV parameters.

The task of computed tomography is the reconstruction of an image given a set of projection measurements called sinogram, in which
the detectors of the CT scanner measure the intensity of attenuated X-ray beams.
Here, we use the linear attenuation model introduced in~\cite{HaMu18}, where the attenuation is proportional to the intersection area of a triangle, 
which is spanned by the X-ray source and a detector element, and the area of an image element.
In detail, the sinogram~$z$ of an image~$x$ is computed by~$z=A_R x$,
where $A_R\in\R^{l\times n}$ is the lookup-table based area integral operator of~\cite{HaMu18} for $R$ angles and $768$ projections implying $l=768\cdot R$.
Typically, a fully sampled acquisition consists of $2304$~angles.
For this task, we consider the problem of angular undersampled CT~\cite{ChTa08}, where only a fraction of the angles are measured.
We use a 4-fold ($R=576$) and 8-fold ($R=288$) angular undersampling to reconstruct a representative image of the MAYO dataset~\cite{McBa17} with $n=768\times768$.
To account for an imbalance of regularization and data fidelity, we manually scale the data fidelity term by~$\lambda>0$, i.e. 
\[
\mathrm{D}(x,z)\coloneqq\frac{\lambda}{2}\Vert A_Rx-z\Vert_2^2.
\]
The resulting smooth variational problem is optimized using accelerated gradient descent with Lipschitz backtracking using $1000$~steps as discussed in~\cite{ChPo16}.

We present qualitative and quantitative results for CT reconstruction in Figure~\ref{fig:CT} for a single abdominal CT image.
As an initialization, we perform $50$~steps of a conjugate gradient method on the data fidelity term (first and last column).
Using the proposed TDV$_{3,25}^3$ regularizer, we are able to suppress the undersampling artifacts while preserving, for instance, the fine vessels in the liver.
This highlights that the learned regularizer can be effectively applied as a generic regularizer for linear inverse problems without any transfer learning,
which is a particular benefit of the variational structure of the proposed approach.

\begin{figure}
\includegraphics[width=\linewidth]{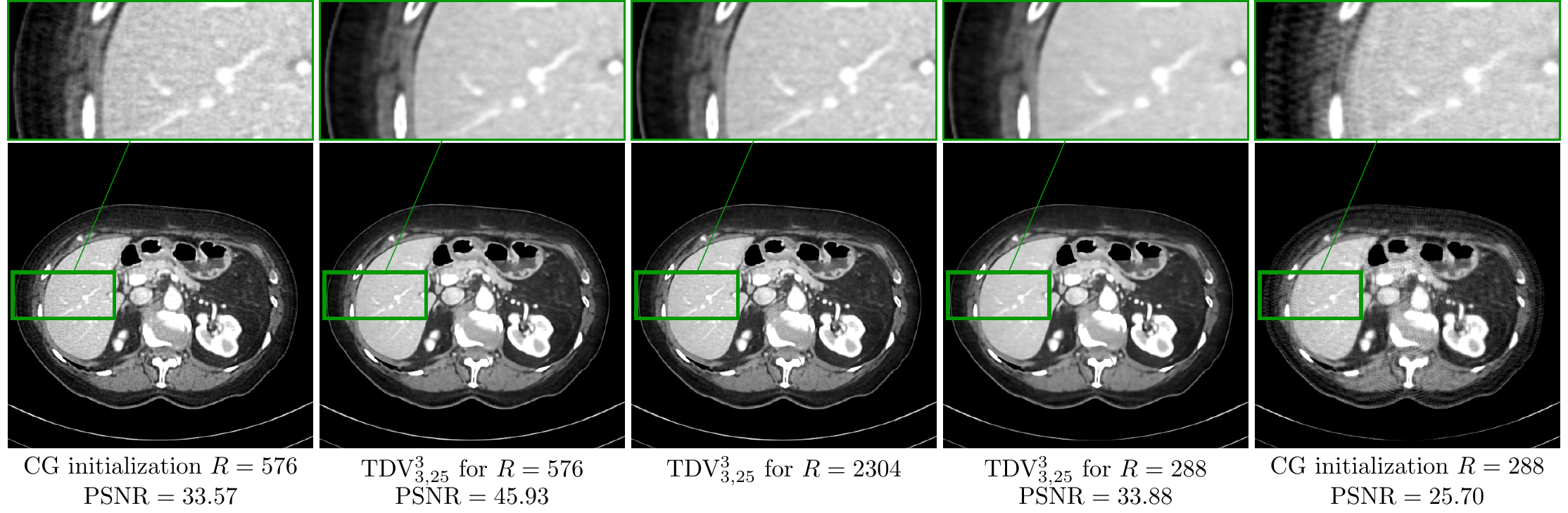}
\caption{Conjugate gradient reconstruction for 4/8-fold angular undersampled CT task (first/fifth image),
results obtained by using the TDV$_{3,25}^3$ regularizer with $\lambda=500\cdot10^3$ for 4-fold (second image) and $\lambda=10^6$ for 8-fold undersampling (fourth image), and fully sampled reference reconstruction using the TDV$^3_{3,25}$ regularizer with $\lambda=125\cdot10^3$ (third image).}
\label{fig:CT}
\end{figure}

\subsection{Magnetic Resonance Imaging Reconstruction}
In what follows, the flexibility of our regularizer TDV$_{3,25}^3$ learned for denoising and $S=10$ is shown for accelerated magnetic resonance imaging (MRI)
\emph{without} any further adaption of~$\theta$.

In accelerated MRI, k--space data $\{z_i\}_{i=1}^{N_C}\subset\C^n$ is acquired using $N_C$ parallel coils, each measuring a fraction of the full k--space to reduce acquisition time~\cite{HaKl18}.
Here, we use the data fidelity term
\[
\mathrm{D}(x,\{z_i\}_{i=1}^{N_C})=\frac{\lambda}{2}\sum_{i=1}^{N_C}\Vert M_RFC_ix-z_i\Vert_2^2,
\]
where $\lambda>0$ is a manually adjusted weighting parameter,
$M_R\in\R^{n\times n}$ is a binary mask for $R$-fold Cartesian undersampling, $F\in\C^{n\times n}$ is the discrete Fourier transform, and $C_i\in\C^{n\times n}$ are sensitivity maps, which are estimated using ESPIRiT~\cite{UeLa14}.
For further details we refer the reader to~\cite{HaKl18}.
We use 4-fold and 6-fold Cartesian undersampled MRI data to reconstruct a sample knee image.
Again, we minimize the resulting variational energy by accelerated gradient descent with Lipschitz backtracking using $1000$~steps~\cite{ChPo16}.

We perform an evaluation of the proposed approach on a representative slice of an undersampled MRI knee acquisition.
The slice has a resolution of $n=320\times320$ and $N_C=15$ receiver coils were used during the acquisition.
Figure~\ref{fig:MRI} depicts qualitative results and $\PSNR$ values for the reconstruction of 4-fold and 6-fold undersampled k--space data.
The first and last columns show the initial images obtained by applying the adjoint operator to the undersampled data.
In the second and fourth column we depict the results obtained using TDV$_{3,25}^3$.
Although the TDV regularizer was not trained to account for undersampling artifacts, almost all artifacts are removed in the reconstructions and only some details in the bone are lost.
This highlights the versatility and effectiveness of the proposed TDV regularizer since both CT and MRI reconstruction can be properly addressed \emph{without} any fine-tuning of the learned parameters.

\begin{figure}
\includegraphics[width=\linewidth]{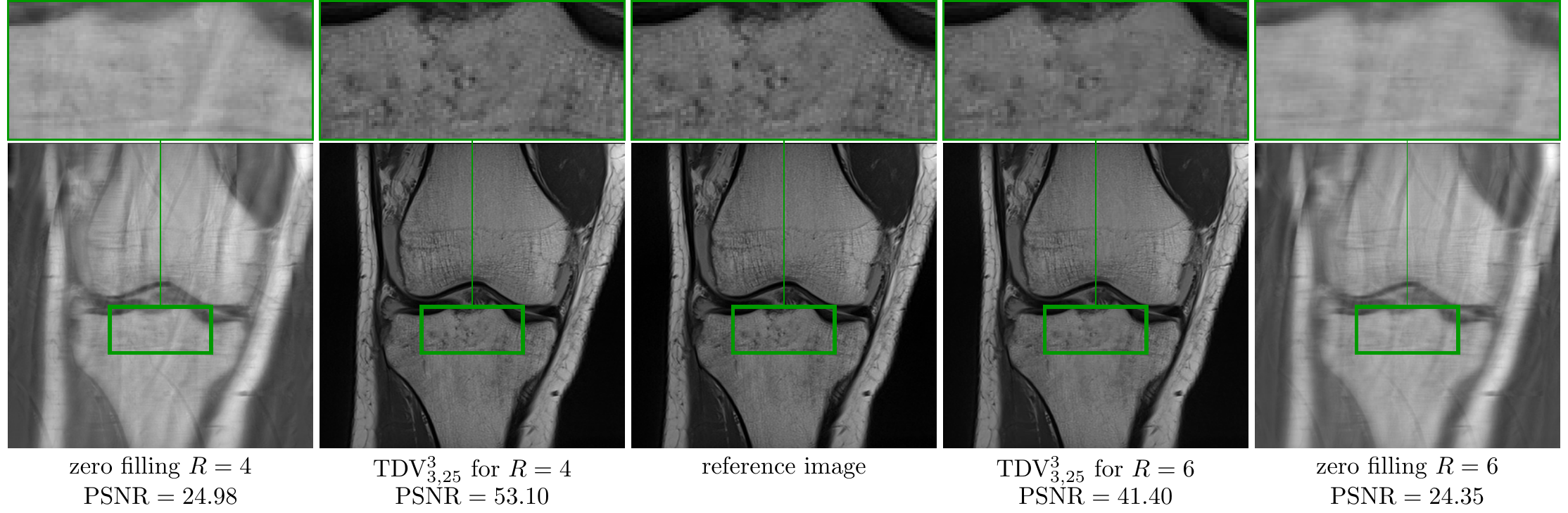}
\caption{Zero filling initialization for acceleration factors~$R\in\{4,6\}$ (first/fifth image), output using the TDV$^3_{3,25}$ regularizer with $\lambda=1000$ for $R=4$ (second image) and $\lambda=1500$ for $R=6$ (fourth image), and fully sampled reference (third image).}
\label{fig:MRI}
\end{figure}

\subsection{Single Image Super-resolution}
In this subsection, we present numerical results for single image super-resolution (SISR).
Here, the linear operator~$A\in\R^{nC/\gamma^2\times nC}$ is given as a downsampling operator, where $\gamma\in\{2,3,4\}$ denotes the scale factor.
In detail, its adjoint operator coincides with MATLAB\textsuperscript{\textregistered}'s bicubic upsampling operator \texttt{imresize},
which is an implementation of a scale factor-dependent interpolation convolution kernel in conjunction with a stride.
Since this restoration problem substantially differs from Gaussian image denoising, the parameters of the TDV regularizer have to be optimized for this task individually.

Let $nC$ be a multiple of $\gamma^2$ and $y\in\R^{nC}$ be a full resolution ground truth image patch uniformly drawn from the BSDS400 dataset.
The observations $z=Ay+\xi\in\R^{nC/\gamma^2}$ used for training are corrupted by additive Gaussian noise~$\xi$ with $\sigma\in\{0,7.65\}$.
For the initialization we set $A_\init=\gamma A^\top$.
The proximal map $(\Id+\frac{T}{S}A^\top A)^{-1}$ is efficiently computed in Fourier space as advocated in~\cite{ZhWe16}.
Here, all results are obtained by training a TDV$_3^3$ regularizer for each scale factor individually.

We compare our SISR results with numerous state-of-the-art networks of similar complexity and list expected $\PSNR$ values of the Y-channel in the YCbCr color space over test datasets in Table~\ref{tab:SR}.
For the BSDS100 dataset, our proposed method achieves similar results as OISR-LF-s~\cite{HeMo19} with only one third of the trainable parameters.
Figure~\ref{fig:OptimalStoppingSequenceSR} depicts a restored sequence of images for SISR with scale factor~$4$ using TDV$_3^3$ for a representative sample image of the Set14 dataset.
Starting from the low resolution initial image, interfaces are gradually sharpened and the best quality is achieved for~$T=T^\ast$.
Beyond this point, interfaces are artificially intensified.
\begin{figure}
\includegraphics[width=\linewidth]{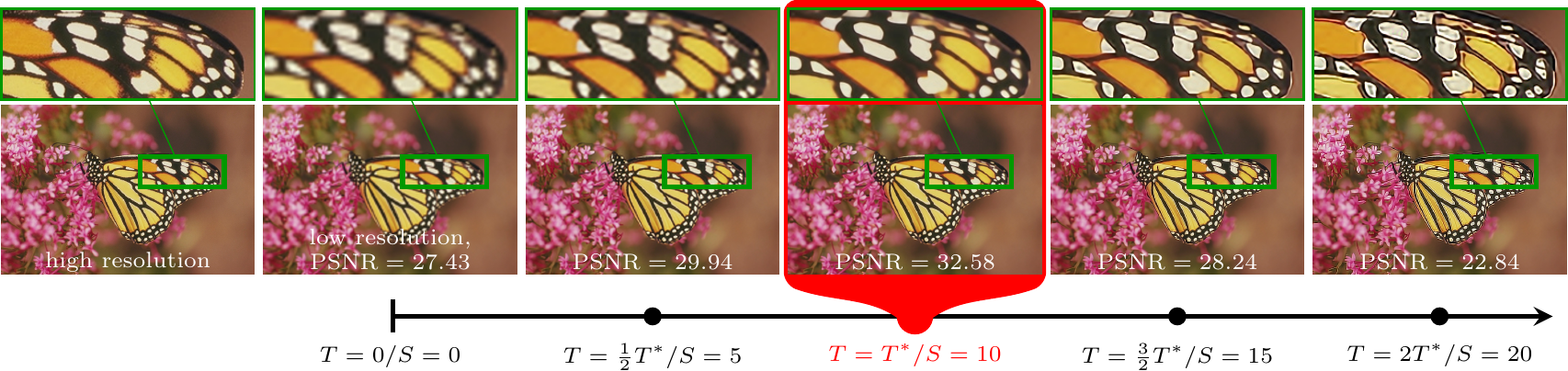}
\caption{
From left to right: High resolution ground truth image, low resolution image with~$\gamma=4$ and $\sigma=0$, and resulting output of TDV$_3^3$ for $(S,T)\in\{(5,\frac{1}{2}T^\ast),(10,T^\ast),(15,\frac{3}{2}T^\ast),(20,2T^\ast)\}$,
where the optimal stopping time is $T^\ast=0.043$.
Note that the best image is framed in red.
}
\label{fig:OptimalStoppingSequenceSR}
\end{figure}
\begin{table}
\caption{$\PSNR$ values of various state-of-the art networks for single image super-resolution ($\sigma=0$) with a comparable number of parameters.}
\label{tab:SR}
\centering
\resizebox{\linewidth}{!}{
\begin{tabular}{l c c*{4}{c} c}
\toprule[1.5pt]
Dataset & Scale & MemNet~\cite{TaYa17} & VDSR~\cite{KiLe16} & DnCNN-3~\cite{ZhZu17} & DRRN~\cite{TaYa17a} & OISR-LF-s~\cite{HeMo19} & TDV$_3^3$ \\ \midrule[1pt]
\multirow{3}{*}{Set14} 
& $\times2$& 33.28 & 33.03 & 33.03 & 33.23 & 33.62 & 33.35\\
& $\times3$& 30.00 & 29.77 & 29.81 & 29.96 & 30.35 & 29.94\\
& $\times4$& 28.26 & 28.01 & 28.04 & 28.21 & 28.63 & 28.41\\ \midrule
\multirow{3}{*}{BSDS100} 
& $\times2$ & 32.08 & 31.90 & 31.90 & 32.05 & 32.20 & 32.17\\
& $\times3$ & 28.96 & 28.82 & 28.85 & 28.95 & 29.11 & 28.96\\
& $\times4$ & 27.40 & 27.29 & 27.29 & 27.38 & 27.60 & 27.55\\ \midrule[1pt]
\# Parameters & & 585,435 & 665,984 & 666,561 & 297,000& 1,370,000 & 387,970\\
\bottomrule[1.5pt]
\end{tabular}
}
\end{table}

\subsection{Eigenfunction Analysis}
To get a better understanding of the local behavior of the proposed TDV regularizer, we perform a nonlinear eigenfunction analysis~\cite{Gi18}.
To this end, we compute \emph{nonlinear eigenfunctions} by minimizing the variational problem
\begin{equation}
\min_{x\in[0,1]^{nC}}\tfrac{1}{2}\Vert D_1\mathrm{R}(x,\theta)-\Lambda(x)x\Vert_2^2,
\label{eq:eigenfunctionAnalysis}
\end{equation}
where the generalized Rayleigh quotient defining the \emph{eigenvalues} is given by
\[
\Lambda(x)=\frac{\langle D_1\mathrm{R}(x,\theta),x\rangle}{\Vert x\Vert_2^2}.
\]
Note that \eqref{eq:eigenfunctionAnalysis} enforces $D_1\mathrm{R}(x,\theta)\approx\Lambda(x)x$ for images with range space~$[0,1]$.
We use Nesterov's projected accelerated gradient descent~\cite{Ne83} to perform the optimization in~\eqref{eq:eigenfunctionAnalysis}.
Due to the nonconvexity of this minimization problem the resulting eigenfunctions strongly depend on the initialization.

Figure~\ref{fig:Eigenfunctions} depicts six images from the BSDS400 dataset (first row), which are used as the initialization, along with their eigenfunctions for gray-scale denoising (second row, $\sigma=25$), color denoising (third row, $\sigma=25$) and SISR (last row, $\gamma=2$ and $\sigma=0$).
For denoising (second and third row), the generated eigenfunctions are composed of piecewise constant regions with smooth edges resulting in cartoon-like simplifications and contrast enhancement (see e.g.~second/third column).
Textured regions of the initial image are transformed into repetitive structures such as stripes and dots.
In contrast, the eigenfunctions for SISR (fourth column) exhibit fine-scaled texture details, which explain the property of the learned regularizer to recover high-frequencies.
These results clearly demonstrate that the learned TDV regularizers are discriminative priors as they adapt to specific image reconstruction tasks.

\begin{figure}
\includegraphics[width=\linewidth]{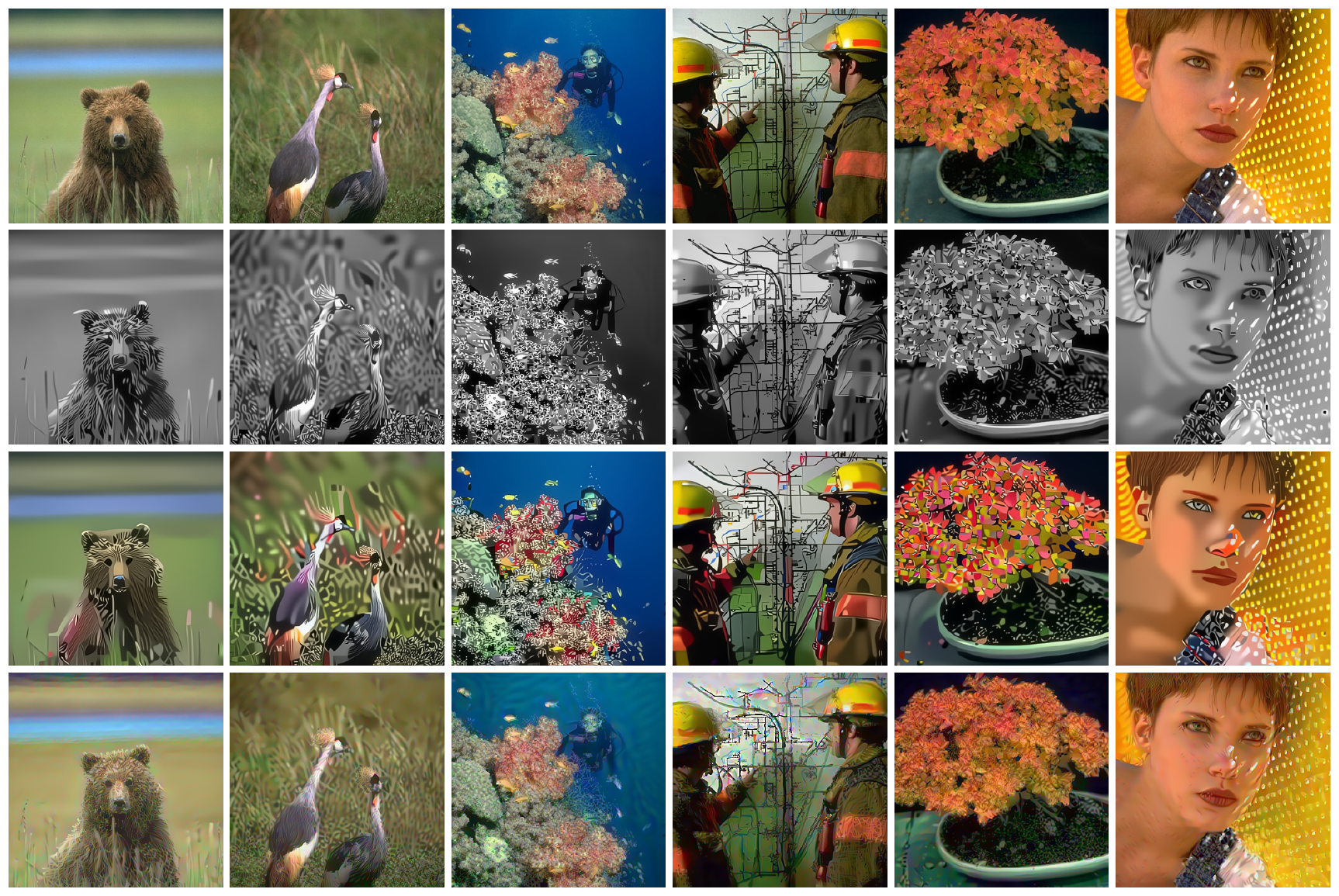}
\caption{
First row: initial images taken from BSDS400 dataset.
Second to fourth row: corresponding eigenfunctions for gray-scale denoising (second row, $\sigma=25$), color denoising (third row, $\sigma=25$) and SISR (last row, $\gamma=2$ and $\sigma=0$).
}
\label{fig:Eigenfunctions}
\end{figure}

\subsection{Stability Analysis}\label{sub:stabilityExperiments}
In what follows, we elaborate on the stability of the proposed approach w.r.t.~perturbations of the initial image and the learned parameters of TDV.
To this end, we numerically analyze the local structure of the regularization energy and experimentally validate Theorem~\ref{thm:stability} and Theorem~\ref{thm:stabilityParameters}.
In all experiments in this subsection, we use the TDV$_3^3$ regularizer.

Let $x\in\R^{nC}$ be an image, $\xi\sim\mathcal{N}(0,\sigma^2\Id)$ and $\theta$ parameters trained for gray-scale Gaussian denoising with $\sigma=25$.
Figure~\ref{fig:DenoisingEnergyLandscape} visualizes the surface plots of the pointwise deep variation $[-1,1]\ni(\zeta_1,\zeta_2)\mapsto\mathrm{r}(\zeta_1 x+\zeta_2\xi,\theta)_i$ as a function of the contrast~$\zeta_1$ and the noise level~$\zeta_2$ for four prototypic pixels~$i$ marked in red.
All surface plots exhibit distinct global minima and no high-frequency oscillations can be observed.
Moreover, the pointwise deep variation strictly increases from the origin in all directions.
\begin{figure}
\centering
\includegraphics[width=.7\linewidth]{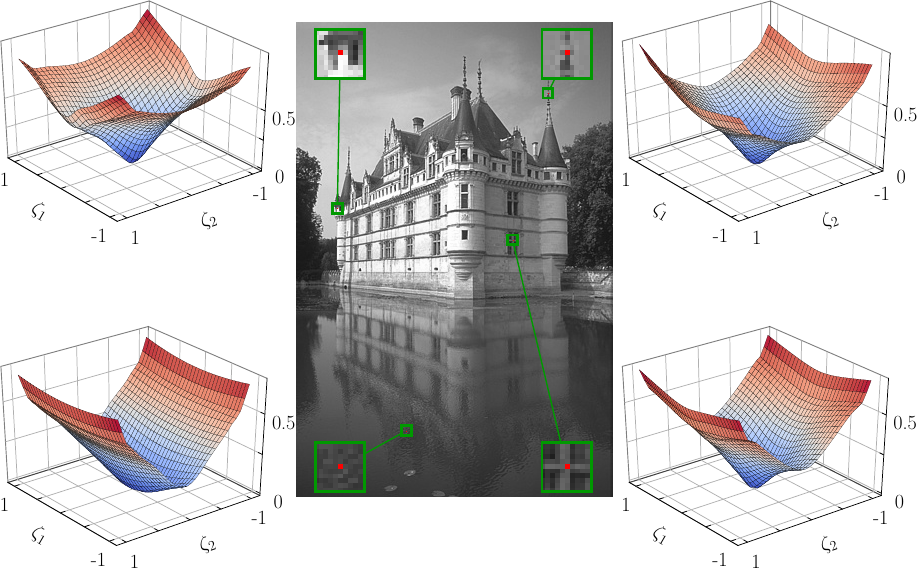}
\caption{Surface plots of the pointwise deep variation $[-1,1]\ni(\zeta_1,\zeta_2)\mapsto\mathrm{r}(\zeta_1 x+\zeta_2\xi,\theta)_i$
of four patches--each evaluated at the red center pixel.}
\label{fig:DenoisingEnergyLandscape}
\end{figure}

Motivated by the aforementioned surface plots, we can now conduct the stability analysis w.r.t.~perturbations of the input.
For this purpose, we estimate quantiles of the local Lipschitz constants~$L_x$ and~$L_\theta$ of TDV by uniformly drawing~$10^5$ patches of size~$128\times 128$ from the BSDS400 dataset.
To this end, we consider an image patch~$y$ randomly drawn from the BSDS68 dataset and let $\xi,\widetilde{\xi}$ be two noise instances independently drawn from $\mathcal{N}(0,\sigma^2\Id)$ for $\sigma=25$.
The associated observations are denoted by $z=y+\xi$ and $\widetilde{z}=y+\widetilde{\xi}$, resulting in the states~$x_s$ and~$\widetilde{x}_s$.
Figure~\ref{fig:stabilitySensitivity} (upper left) depicts the normalized norm differences~$\frac{1}{nC}\Vert x_s-\widetilde{x}_s\Vert_2$ for all $68$~patches (light blue curves), the corresponding mean curve (blue curve)
as well as the upper bounds obtained from Theorem~\ref{thm:stability} for $\delta=0.5$ (orange curve) and $\delta=0.05$ (green curve) for gray-scale additive Gaussian denoising as a function of~$s$.
It turns out that the normalized norm differences along the trajectories are only slightly overestimated.
Furthermore, the normalized norm differences strictly monotonically decrease for increasing~$s$, which is also reflected in the upper bounds due to $\alpha_1(\delta)<1$.

Next, we elaborate on the stability analysis w.r.t.~variations of the learned parameters.
Let $y$ be a randomly drawn $128\times128$-patch from the BSDS68 dataset, which is corrupted by additive Gaussian noise~$\xi$ with~$\sigma=25$, i.e.~$z=y+\xi$.
We consider optimized parameters~$\theta\in\Theta$ for gray-scale Gaussian denoising.
The TDV parameters~$\widetilde{\theta}$ satisfies $\widetilde{\theta}\sim\mathcal{U}(\proj_\Theta(B_{\epsilon}(\theta)))$ with $\epsilon=0.1$.
Hence, $\widetilde{\theta}$ is the element-wise sum of $\theta$ and strong uniform noise in the relative $\epsilon$-ball around~$\theta$.
We denote by~$x_s$ and~$\widetilde{x}_s$ two states associated with~$\theta$ and~$\widetilde{\theta}$ emanating from the same noisy observation~$z$.
In Figure~\ref{fig:stabilitySensitivity} (lower left), the normalized norm differences of the states~$x_s$ and $\widetilde{x}_s$ for all $68$~patches (light blue curves), the corresponding mean curve (blue curve)
as well as the theoretical upper bounds derived in Theorem~\ref{eq:stabilityParameters} for $\delta=0.5$ (orange curve) and $\delta=0.05$ (green curve) are plotted as a function of~$s$.
As a result, the upper bound for~$\delta=0.5$ for $s=10$ is roughly four times higher than the expected curve of the normalized norm differences.

In the case of SISR for~$\gamma=3$ and~$\sigma=7.65$, the results of the stability analysis are depicted on the right side of Figure~\ref{fig:stabilitySensitivity}.
Here, the theoretical upper bounds are less tight compared to Gaussian denoising due to the inclusion of the non-trivial linear operators~$A$ with a non-empty nullspace and $A_\init$ with $\Vert A_\init\Vert_2=\gamma$.
Note that the solutions of SISR are not unique due to the structure of~$A$.

To conclude, in all four cases the normalized norm differences (blue curves) are almost flat and the band width only slightly increases with~$s$.
This numerically validates that the proposed method is robust w.r.t.~variations of both the observations and the learned parameters.

\begin{figure}
\centering
\includegraphics[width=\linewidth]{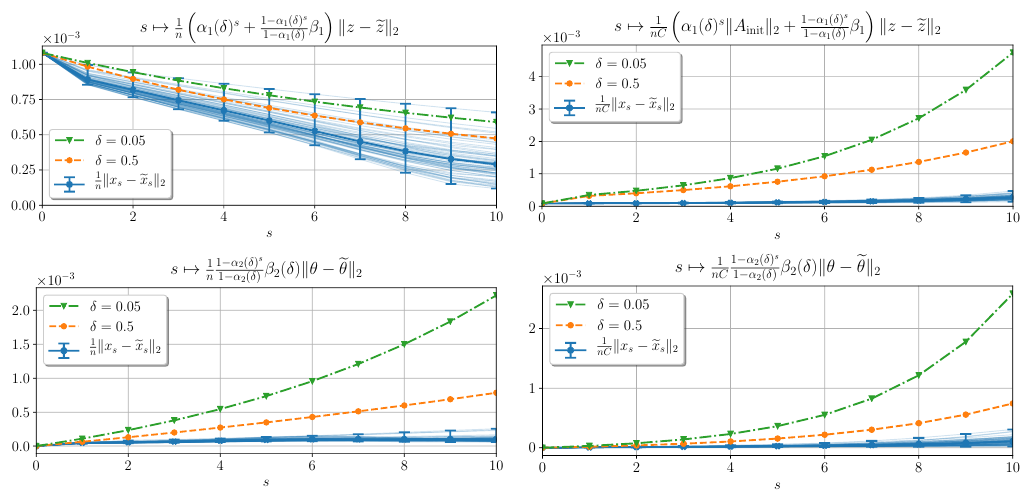}

\caption{First row: stability analysis w.r.t.~input of the proposed approach for gray-scale Gaussian denoising (left) trained for $\sigma=25$/$S=10$ and SISR (right) with $\gamma=3$/$\sigma=7.65$/$S=10$.
Second row: stability analysis w.r.t.~the learned TDV$_3^3$ parameters for gray-scale Gaussian denoising (left) and SISR (right).}
\label{fig:stabilitySensitivity}
\end{figure}

\subsection{Robustness against Adversarial Attacks}
In what follows, we numerically check the robustness of the proposed method against adversarial attacks of the TDV$_3^3$ regularizer trained for gray-scale Gaussian denoising with~$\sigma=25$.
Let $y\in\R^{nC}$ be a ground truth image patch, $\xi\sim\mathcal{N}(0,\sigma^2\Id)$ Gaussian noise and $z=y+\xi\in\R^l$.
The adversarial noise~$\widetilde{\xi}$ for $\epsilon>0$ is computed via
\[
\max_{\widetilde{\xi}\in\R^l:\Vert\widetilde{\xi}\Vert_2\leq\epsilon}\Vert\widehat{x}_S(y,\xi+\widetilde{\xi},T,\theta)-y\Vert_2^2.
\]
Thus, we seek the noise structure that leads to the largest deviation from~$y$ around~$z$.

Figure~\ref{fig:adversarial} shows two different ground truth image patches, the corresponding restored images along with the computed adversarial noise structures and corresponding output images for~$\epsilon\in\{1,2\}$.
As a result, with increasing~$\epsilon$ high-frequency patterns (dotted and striped structures) are generated in the adversarial noise, which are emphasized in the corresponding output images.
Moreover, we observe high-frequency local noise patterns in both reconstructed images for~$\epsilon=2$.
In particular, no new structures are hallucinated, only existing patterns are intensified in~$x_S(\widetilde{\xi})$.

\begin{figure}
\includegraphics[width=\linewidth]{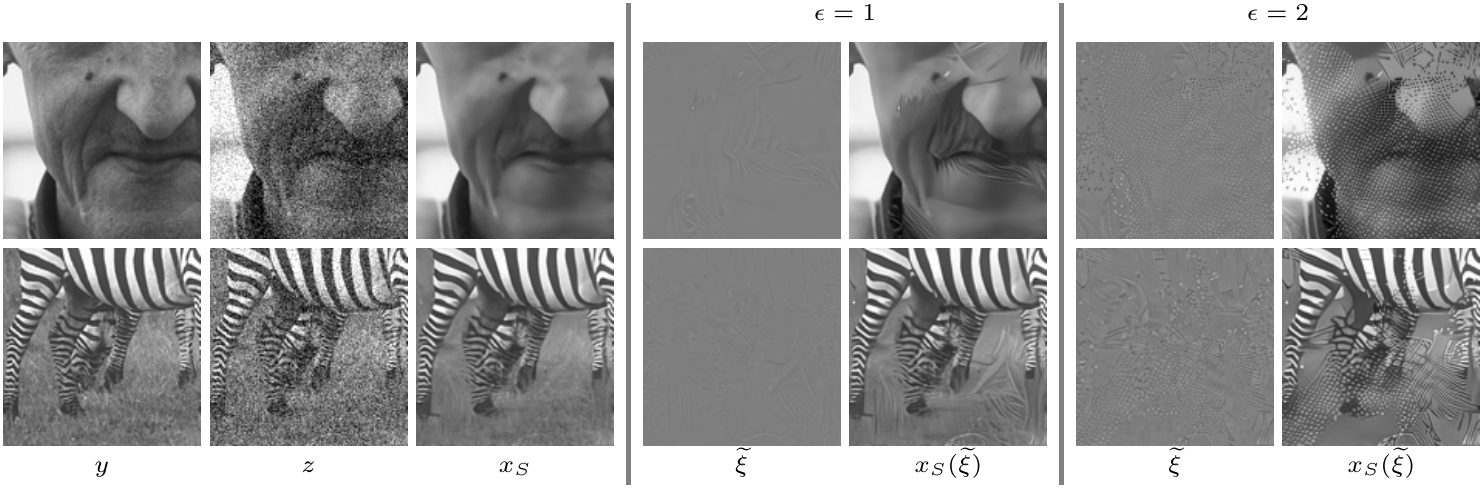}

\caption{From left to right: ground truth image patch~$y$ (first column), noisy observation~$z=y+\xi$ (second column), reconstructed image (third column),
pairs of (adversarial) noise and resulting output for radii~$\epsilon=1$ (fourth/fifth column) and $\epsilon=2$ (sixth/seventh column), where $x_S(\widetilde{\xi})=\widehat{x}_S(y,\xi+\widetilde{\xi},T,\theta)$.
The adversarial noise~$\widetilde{\xi}$ is displayed in the range~$[-0.5,0.5]$.}
\label{fig:adversarial}
\end{figure}

\subsection{Empirical Upper Bound for Generalization Error}
Next, we experimentally compute worst case upper bounds for the generalization error of the TDV$_3^3$ regularizer trained for gray-scale Gaussian denoising with $\sigma=25$.
As a starting point, let $\mathcal{Y}\subset[0,1]^n$ be the set of natural images with distribution~$\mathcal{T}_\mathcal{Y}$.
Further, let~$\mathcal{Y}'\subset\mathcal{Y}$ be a collection of $10^5$ ground truth image patches of size $128\times 128$ randomly drawn from the BSDS400 dataset, the BSDS68 dataset, and the DIV2K validation set~\cite{AgTi17}.
The uniform distribution on $\mathcal{Y}'$ is denoted by $\mathcal{T}_{\mathcal{Y}'}$.
Following~\cite{BoEl02}, the \emph{empirical risk} w.r.t.~$\mathcal{Y}'$ is defined as
\[
E_{\emp}(\mathcal{Y}')\coloneqq\frac{1}{\vert\mathcal{Y}'\vert}\sum_{y'\in\mathcal{Y}'}\loss(\widehat{x}_S(y',\xi_{y'})-y'),
\]
the \emph{expected loss} of~$\mathcal{Y}$ reads as
\[
E(\mathcal{Y})\coloneqq\E_{y\sim\mathcal{T}_\mathcal{Y}}\loss(\widehat{x}_S(y,\xi_y)-y),
\]
where $\loss(x)=\Vert x\Vert_2^2$, $\xi_y$ and $\xi_{y'}$ are a priori sampled noise instances drawn from~$\mathcal{N}(0,\sigma^2\Id)$, and $\widehat{x}_S(y,\xi)=\widehat{x}_S(y,\xi,T,\theta)$.
In this case, the \emph{generalization error} for $\mathcal{Y}'$ is defined as $\vert E(\mathcal{Y})-E_{\emp}(\mathcal{Y}')\vert$.
A worst case upper bound for this generalization error is given by
\begin{equation}\label{eq:generalizationWorstCase}
\max_{\widetilde{y}\in[0,1]^n}\loss(\widehat{x}_S(\widetilde{y},\xi)-\widetilde{y})-E_{\emp}(\mathcal{Y}')
\end{equation}
for an a priori sampled $\xi\sim\mathcal{N}(0,\sigma^2\Id)$ since the expected loss is estimated from above by its single worst realization.

To analyze the dependency between~$\loss(\widehat{x}_S(y',\xi_{y'})-y')$ and the regularization energy $\mathrm{R}(y',\theta)$, we show a scatter plot of the corresponding plane in Figure~\ref{fig:generalization} (left) for all $y'\in\mathcal{Y}'$.
We observe a strikingly linear dependency, which is reflected by an $R^2$-value of~$0.985$ of a linear regression with intercept.
This linear dependency gives rise to a probabilistic analysis of worst case upper bounds for the generalization error on quantiles of the regularization energy.
For this reason, we define the cumulative distribution function
\[
F_\mathrm{R}(H)=\Prob\left(\mathrm{R}(y',\theta)\leq H:y'\sim\mathcal{T}_{\mathcal{Y}'}\right)
\]
for $H\in\R$.
Note that $F_\mathrm{R}^{-1}(q)$ for $q\in(0,1]$ defines the $q^{th}$-quantile of the regularization energy over~$\mathcal{Y}'$.
Then, we derive an upper bound for the generalization error restricted to the subset
\[
\mathcal{Y}'_q=\{y'\in\mathcal{Y}':\mathrm{R}(y',\theta)\leq F_\mathrm{R}^{-1}(q)\}
\]
for $q\in(0,1]$.
In this setting, the expected loss of the $q^{th}$-quantile is estimated from above by $\loss(\widehat{x}_S(\widetilde{y}_q)-\widetilde{y}_q)$, where
\begin{equation}
\widetilde{y}_q\in\argmax_{y\in[0,1]^n}\loss(\widehat{x}_S(y,\xi)-y)\quad\text{s.t.}\quad\mathrm{R}(y)\leq F_\mathrm{R}^{-1}(q).
\label{eq:generalization}    
\end{equation}
In detail, we try to identify the image patch~$\widetilde{y}_q$ that leads to the worst case loss~$\loss$ among all image patches in $[0,1]^n$ such that their regularization energy~$\mathrm{R}(\widetilde{y}_q)$ is at most $F_\mathrm{R}^{-1}(q)$.
Hence, an upper bound for the generalization error on the set~$\mathcal{Y}'_q$ is given by
\[
\mathrm{G}(q)\coloneqq\loss(\widehat{x}_S(\widetilde{y}_q,\xi)-\widetilde{y}_q)-E_{\emp}(\mathcal{Y}'_q).
\]
To compute the worst case ground truth image, we account for the constraint in~\eqref{eq:generalization} by a quadratic barrier approach and solve the resulting minimization problem using Nesterov's accelerated gradient method~\cite{Ne83} with Lipschitz backtracking starting from a patch with uniform noise.

Figure~\ref{fig:generalization} (second plot) depicts the semi-logarithmic plots of the upper bound for the empirical risk (blue curve), the empirical risk (orange curve), and the upper bound for the generalization error (green curve) on the set~$\mathcal{Y}'_q$ depending on~$q$.
For convenience, the third plot depicts corresponding $\PSNR$ curves measured in dB.
We note that the upper bound for the generalization error slightly increases with larger regularization energy values represented by larger~$q$.
We observe that the upper bound is not tight, which originates from the minimization in~\eqref{eq:generalization} among all patches in $[0,1]^n$.
The computed worst case patches~$\widetilde{y}_q$ along with the reconstructed output images~$\widehat{x}_S(\widetilde{y}_q,\xi)$ are shown in Figure~\ref{fig:generalization} (right).
For low values of~$q$ the worst case ground truth patches consist of high-oscillatory stripe patterns and checkerboard artifacts, whereas noise and texture patterns are dominant for higher values of~$q$.
We emphasize that all generated patches are artificial and not likely to be contained in any natural image, that is why the upper bound for generalization error has a tendency to overestimate the actual generalization error.

\begin{figure}
\includegraphics[width=\linewidth]{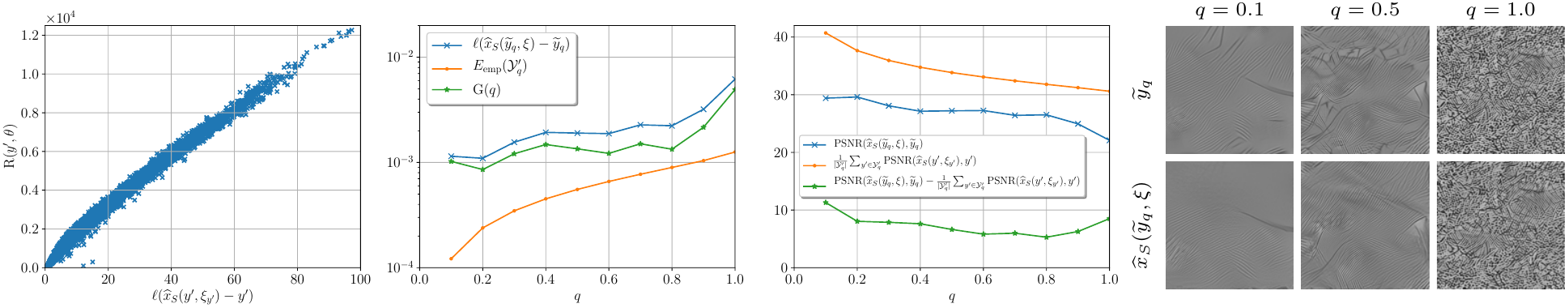}

\caption{First plot: scatter plot in the $\loss(\widehat{x}_S(y,\xi_y)-y)$ and $\mathrm{R}(y,\theta)$-plane.
Second plot: semi-logarithmic plots of the upper bound for expected loss (blue curve), the empirical risk (orange curve) and the upper bound for the generalization error (green curve) restricted to $\mathcal{Y}'_q$ as a function of the quantiles~$q$.
Third plot: corresponding $\PSNR$ curves measured in dB.
Right: pairs of worst case ground truth $\widetilde{y}_q$ and corresponding output image $\widehat{x}_S(\widetilde{y}_q,\xi)$ for $q=0.1,0.5,1.0$.
}
\label{fig:generalization}
\end{figure}

\section{Conclusion}
The proposed total deep variation regularizer is motivated by established deep network architectures.
Moreover, the inherent variational structure of our approach enables a rigorous mathematical understanding encompassing
an optimality condition for optimal stopping, and a nonlinear eigenfunction analysis.
We have derived theoretical upper bounds for the stability analysis, which led to relatively tight bounds in the numerical experiments.
For image denoising and single image super-resolution, our model generates state-of-the-art results with an impressively low number of trainable parameters.
To underline the versatility of TDV for generic linear inverse problems, we successfully demonstrated their applicability for the challenging CT and MRI reconstruction tasks without requiring any additional training.
Finally, we have conducted adversarial attacks and an empirical worst case generalization error analysis to demonstrate the robustness of our approach.

\section*{Acknowledgments}
The authors acknowledge support from the ERC starting grant HOMOVIS (No. 640156) and ERC advanced grant OCLOC (No. 668998).

\bibliographystyle{plain}
\bibliography{references.bib}

\begin{thebibliography}{10}

\bibitem{AgTi17}
Eirikur Agustsson and Radu Timofte.
\newblock {NTIRE} 2017 challenge on single image super-resolution: Dataset and
  study.
\newblock In {\em IEEE Conference on Computer Vision and Pattern Recognition
  Workshops}, 2017.

\bibitem{AmGi08}
Luigi Ambrosio, Nicola Gigli, and Giuseppe Savar\'{e}.
\newblock {\em Gradient flows in metric spaces and in the space of probability
  measures}.
\newblock Lectures in Mathematics ETH Z\"{u}rich. Birkh\"{a}user Verlag, Basel,
  second edition, 2008.

\bibitem{AnRe20}
Vegard Antun, Francesco Renna, Clarice Poon, Ben Adcock, and Anders~C. Hansen.
\newblock On instabilities of deep learning in image reconstruction and the
  potential costs of ai.
\newblock {\em Proceedings of the National Academy of Sciences}, 2020.

\bibitem{BoEl02}
Olivier Bousquet and Andr\'{e} Elisseeff.
\newblock Stability and generalization.
\newblock {\em J. Mach. Learn. Res.}, 2(3):499--526, 2002.

\bibitem{BrKu10}
Kristian Bredies, Karl Kunisch, and Thomas Pock.
\newblock Total generalized variation.
\newblock {\em SIAM J. Imaging Sci.}, 3(3):492--526, 2010.

\bibitem{ChLi97}
Antonin. Chambolle and Pierre-Louis Lions.
\newblock Image recovery via total variation minimization and related problems.
\newblock {\em Numer. Math.}, 76(2):167--188, 1997.

\bibitem{ChPo16}
Antonin Chambolle and Thomas Pock.
\newblock An introduction to continuous optimization for imaging.
\newblock {\em Acta Numer.}, 25:161--319, 2016.

\bibitem{ChPo19}
Antonin Chambolle and Thomas Pock.
\newblock Total roto-translational variation.
\newblock {\em Numer. Math.}, 142(3):611--666, 2019.

\bibitem{ChKa02}
Tony~F. Chan, Sung~Ha Kang, and Jianhong Shen.
\newblock Euler's elastica and curvature-based inpainting.
\newblock {\em SIAM J. Appl. Math.}, 63(2):564--592, 2002.

\bibitem{ChTa08}
Guang-Hong Chen, Jie Tang, and Shuai Leng.
\newblock Prior image constrained compressed sensing ({PICCS}): A method to
  accurately reconstruct dynamic {CT} images from highly undersampled
  projection data sets.
\newblock {\em Medical Physics}, 35(2):660--663, 2008.

\bibitem{ChPo17}
Yunjin Chen and Thomas Pock.
\newblock Trainable nonlinear reaction diffusion: A flexible framework for fast
  and effective image restoration.
\newblock {\em IEEE Transactions on Pattern Analysis and Machine Intelligence},
  39(6):1256--1272, 2017.

\bibitem{DaFo07}
Kostadin Dabov, Alessandro Foi, Vladimir Katkovnik, and Karen Egiazarian.
\newblock Image denoising by sparse {3-D} transform-domain collaborative
  filtering.
\newblock {\em IEEE Transactions on Image Processing}, 16(8):2080--2095, 2007.

\bibitem{Do12}
Justin Domke.
\newblock Generic methods for optimization-based modeling.
\newblock In {\em International Conference on Artificial Intelligence and
  Statistics}, pages 318--326, 2012.

\bibitem{EHa19}
Weinan E, Jiequn Han, and Qianxiao Li.
\newblock A mean-field optimal control formulation of deep learning.
\newblock {\em Res. Math. Sci.}, 6(1):Paper No. 10, 41, 2019.

\bibitem{EfKo19}
Alexander Effland, Erich Kobler, Karl Kunisch, and Thomas Pock.
\newblock Variational {N}etworks: {A}n {O}ptimal {C}ontrol {A}pproach to
  {E}arly {S}topping {V}ariational {M}ethods for {I}mage {R}estoration.
\newblock {\em J. Math. Imaging Vision}, 62(3):396--416, 2020.

\bibitem{Gi18}
Guy Gilboa.
\newblock {\em Nonlinear eigenproblems in image processing and computer
  vision}.
\newblock Advances in Computer Vision and Pattern Recognition. Springer, Cham,
  2018.

\bibitem{GoSh15}
Ian~J Goodfellow, Jonathon Shlens, and Christian Szegedy.
\newblock Explaining and harnessing adversarial examples.
\newblock In {\em International Conference on Learning Representations}, 2015.

\bibitem{HaMu18}
Sungsoo Ha and Klaus Mueller.
\newblock A look-up table-based ray integration framework for {2-D/3-D} forward
  and back projection in {X}-ray {CT}.
\newblock {\em IEEE Transactions on Medical Imaging}, 37:361--371, 2018.

\bibitem{HaKl18}
Kerstin Hammernik, Teresa Klatzer, Erich Kobler, Michael~P. Recht, Daniel~K.
  Sodickson, Thomas Pock, and Florian Knoll.
\newblock Learning a variational network for reconstruction of accelerated
  {MRI} data.
\newblock {\em Magnetic Resonance in Medicine}, 79(6):3055--3071, 2018.

\bibitem{HeMo19}
Xiangyu He, Zitao Mo, Peisong Wang, Yang Liu, Mingyuan Yang, and Jian Cheng.
\newblock {ODE}-inspired network design for single image super-resolution.
\newblock In {\em IEEE Conference on Computer Vision and Pattern Recognition},
  pages 1732--1741, 2019.

\bibitem{HuMu99}
Jinggang Huang and David Mumford.
\newblock Statistics of natural images and models.
\newblock In {\em IEEE Conference on Computer Vision and Pattern Recognition},
  pages 541--547, 1999.

\bibitem{JiLi19}
Xixi Jia, Sanyang Liu, Xiangchu Feng, and Lei Zhang.
\newblock {FOCNet}: A fractional optimal control network for image denoising.
\newblock In {\em IEEE Conference on Computer Vision and Pattern Recognition},
  pages 6047--6056, 2019.

\bibitem{KiLe16}
Jiwon Kim, Jung~Kwon Lee, and Kyoung~Mu Lee.
\newblock Accurate image super-resolution using very deep convolutional
  networks.
\newblock In {\em IEEE Conference on Computer Vision and Pattern Recognition},
  pages 1646--1654, 2016.

\bibitem{KiBa15}
Diederik~P. Kingma and Jimmy~Lei Ba.
\newblock {ADAM}: a method for stochastic optimization.
\newblock In {\em International Conference on Learning Representations}, 2015.

\bibitem{KoEf20}
Erich Kobler, Alexander Effland, Karl Kunisch, and Thomas Pock.
\newblock Total deep variation for linear inverse problems.
\newblock In {\em IEEE Conference on Computer Vision and Pattern Recognition},
  2020.

\bibitem{KoKl17}
Erich Kobler, Teresa Klatzer, Kerstin Hammernik, and Thomas Pock.
\newblock Variational networks: Connecting variational methods and deep
  learning.
\newblock In {\em German Conference on Pattern Recognition}, pages 281--293.
  Springer International Publishing, 2017.

\bibitem{Le16}
Stamatios Lefkimmiatis.
\newblock Non-local color image denoising with convolutional neural networks.
\newblock In {\em IEEE Conference on Computer Vision and Pattern Recognition},
  pages 5882--5891, 2016.

\bibitem{LeNa11}
Anat Levin and Boaz Nadler.
\newblock Natural image denoising: Optimality and inherent bounds.
\newblock In {\em IEEE Conference on Computer Vision and Pattern Recognition},
  pages 2833--2840, 2011.

\bibitem{LiSc20}
Housen Li, Johannes Schwab, Stephan Antholzer, and Markus Haltmeier.
\newblock {NETT:} solving inverse problems with deep neural networks.
\newblock {\em Inverse Problems}, 2020.

\bibitem{LuOk18}
Sebastian Lunz, Ozan {\"O}ktem, and Carola-Bibiane Sch{\"o}nlieb.
\newblock Adversarial regularizers in inverse problems.
\newblock In {\em Advances in Neural Information Processing Systems}, pages
  8507--8516, 2018.

\bibitem{MaFo01}
David Martin, Charless Fowlkes, Doron Tal, and Jitendra Malik.
\newblock A database of human segmented natural images and its application to
  evaluating segmentation algorithms and measuring ecological statistics.
\newblock In {\em IEEE International Conference on Computer Vision}, pages
  416--423, 2001.

\bibitem{McBa17}
Cynthia~H. McCollough, Adam~C. Bartley, Rickey~E. Carter, Baiyu Chen, Tammy~A.
  Drees, Phillip Edwards, David~R. Holmes~III, Alice~E. Huang, Farhana Khan,
  Shuai Leng, Kyle~L. McMillan, Gregory~J. Michalak, Kristina~M. Nunez, Lifeng
  Yu, and Joel~G. Fletcher.
\newblock Low-dose {CT} for the detection and classification of metastatic
  liver lesions: Results of the 2016 low dose {CT} grand challenge.
\newblock {\em Medical Physics}, 44(10):339--352, 2017.

\bibitem{MeMo17}
Tim Meinhardt, Michael M\"oller, Caner Hazirbas, and Daniel Cremers.
\newblock Learning proximal operators: Using denoising networks for
  regularizing inverse imaging problems.
\newblock In {\em IEEE International Conference on Computer Vision}, pages
  1781--1790, 2017.

\bibitem{MoSe17}
Seyed-Mohsen Moosavi-Dezfooli, Alhussein Fawzi, Omar Fawzi, and Pascal
  Frossard.
\newblock Universal adversarial perturbations.
\newblock In {\em IEEE Conference on Computer Vision and Pattern Recognition},
  pages 1765--1773, 2017.

\bibitem{MoSe16}
Seyed-Mohsen Moosavi-Dezfooli, Alhussein Fawzi, and Pascal Frossard.
\newblock Deepfool: a simple and accurate method to fool deep neural networks.
\newblock In {\em IEEE Conference on Computer Vision and Pattern Recognition},
  pages 2574--2582, 2016.

\bibitem{Ne83}
Yuri~E. Nesterov.
\newblock A method of solving a convex programming problem with convergence
  rate $\mathcal{O}(\frac{1}{k^2})$.
\newblock {\em Dokl. Akad. Nauk SSSR}, 269(3):543--547, 1983.

\bibitem{NgJo02}
Andrew~Y. Ng and Michael~I. Jordan.
\newblock On discriminative vs. generative classifiers: A comparison of
  logistic regression and naive {Bayes}.
\newblock In {\em Advances in Neural Information Processing Systems 14}, pages
  841--848. MIT Press, 2002.

\bibitem{NiMu93}
Mark Nitzberg, David Mumford, and Takahiro Shiota.
\newblock {\em Filtering, segmentation and depth}, volume 662 of {\em Lecture
  Notes in Computer Science}.
\newblock Springer-Verlag, Berlin, 1993.

\bibitem{PlRo18}
Tobias Pl\"{o}tz and Stefan Roth.
\newblock Neural nearest neighbors networks.
\newblock In {\em Advances in Neural Information Processing Systems 31}, pages
  1087--1098. Curran Associates, Inc., 2018.

\bibitem{RiCh17}
J.~H. Rick~Chang, Chun-Liang Li, Barn\'abas P\'oczos, B.~V.~K. Vijaya~Kumar,
  and Aswin~C. Sankaranarayanan.
\newblock One network to solve them all --- solving linear inverse problems
  using deep projection models.
\newblock In {\em IEEE International Conference on Computer Vision}, pages
  5888--5897, 2017.

\bibitem{RoEl17}
Yaniv Romano, Michael Elad, and Peyman Milanfar.
\newblock The little engine that could: regularization by denoising ({RED}).
\newblock {\em SIAM J. Imaging Sci.}, 10(4):1804--1844, 2017.

\bibitem{RoFi15}
Olaf Ronneberger, Philipp Fischer, and Thomas Brox.
\newblock {U-Net}: Convolutional networks for biomedical image segmentation.
\newblock In {\em Medical Image Computing and Computer-Assisted Intervention},
  pages 234--241. Springer, 2015.

\bibitem{RoBl09}
Stefan Roth and Michael~J. Black.
\newblock {Fields of Experts}.
\newblock {\em Int. J. Comput. Vis.}, 82(2):205--229, 2009.

\bibitem{RuOsFa92}
Leonid~I. Rudin, Stanley Osher, and Emad Fatemi.
\newblock Nonlinear total variation based noise removal algorithms.
\newblock {\em Phys. D}, 60(1-4):259--268, 1992.

\bibitem{SaTa09}
Kegan G.~G. Samuel and Marshall~F. Tappen.
\newblock Learning optimized {MAP} estimates in continuously-valued {MRF}
  models.
\newblock In {\em IEEE Conference on Computer Vision and Pattern Recognition},
  pages 477--484, 2009.

\bibitem{SzZa14}
Christian Szegedy, Wojciech Zaremba, Ilya Sutskever, Joan Bruna, Dumitru Erhan,
  Ian Goodfellow, and Rob Fergus.
\newblock Intriguing properties of neural networks.
\newblock In {\em International Conference on Learning Representations}, 2014.

\bibitem{TaYa17a}
Ying Tai, Jian Yang, and Xiaoming Liu.
\newblock Image super-resolution via deep recursive residual network.
\newblock In {\em IEEE Conference on Computer Vision and Pattern Recognition},
  pages 2790--2798, 2017.

\bibitem{TaYa17}
Ying Tai, Jian Yang, Xiaoming Liu, and Chunyan Xu.
\newblock {MemNet}: A persistent memory network for image restoration.
\newblock In {\em IEEE International Conference on Computer Vision}, pages
  4549--4557, 2017.

\bibitem{Te12}
Gerald Teschl.
\newblock {\em Ordinary differential equations and dynamical systems}, volume
  140 of {\em Graduate Studies in Mathematics}.
\newblock American Mathematical Society, Providence, RI, 2012.

\bibitem{UeLa14}
Martin Uecker, Peng Lai, Mark~J. Murphy, Patrick Virtue, Michael Elad, John~M.
  Pauly, Shreyas~S. Vasanawala, and Michael Lustig.
\newblock {ESPIRiT}—an eigenvalue approach to autocalibrating parallel {MRI}:
  where {SENSE} meets {GRAPPA}.
\newblock {\em Magnetic Resonance in Medicine}, 71(3):990--1001, 2014.

\bibitem{VeSi13}
Singanallur~V. Venkatakrishnan, Charles~A. Bouman, and Brendt Wohlberg.
\newblock Plug-and-play priors for model based reconstruction.
\newblock In {\em IEEE Global Conference on Signal and Information Processing},
  pages 945--948, 2013.

\bibitem{Ze85}
Eberhard Zeidler.
\newblock {\em Nonlinear functional analysis and its applications. {III}}.
\newblock Springer-Verlag, New York, 1985.

\bibitem{ZhVa20}
Kai Zhang, Luc Van~Gool, and Radu Timofte.
\newblock Deep unfolding network for image super-resolution.
\newblock {\em arXiv:2003.10428}, 2020.

\bibitem{ZhZu17}
Kai Zhang, Wangmeng Zuo, Yunjin Chen, Deyu Meng, and Lei Zhang.
\newblock Beyond a {Gaussian} denoiser: Residual learning of deep {CNN} for
  image denoising.
\newblock {\em IEEE Transactions on Image Processing}, 26(7):3142--3155, 2017.

\bibitem{ZhZu18}
Kai Zhang, Wangmeng Zuo, and Lei Zhang.
\newblock {FFDNet}: Toward a fast and flexible solution for {CNN}-based image
  denoising.
\newblock {\em IEEE Transactions on Image Processing}, 27(9):4608--4622, 2018.

\bibitem{Zh19}
Richard Zhang.
\newblock Making convolutional networks shift-invariant again.
\newblock In {\em International Conference on Machine Learning}, 2019.

\bibitem{ZhWe16}
Ningning Zhao, Qi~Wei, Adrian Basarab, Nicolas Dobigeon, Denis Kouam\'{e}, and
  Jean-Yves Tourneret.
\newblock Fast single image super-resolution using a new analytical solution
  for {$\ell_2$}--{$\ell_2$} problems.
\newblock {\em IEEE Trans. Image Process.}, 25(8):3683--3697, 2016.

\bibitem{ZhWu98}
Song~Chun Zhu, Yingnian Wu, and David Mumford.
\newblock Filters, random fields and maximum entropy ({FRAME}): Towards a
  unified theory for texture modeling.
\newblock {\em Int. J. Comput. Vision}, 27(2):107--126, 1998.

\end{thebibliography}

\end{document}